\documentclass{article}
\usepackage{microtype}
\usepackage{graphicx}
\usepackage{subfig}
\usepackage{booktabs} % for professional tables
\usepackage{url}
\usepackage{hyperref}

\usepackage{derivative}
\usepackage{amssymb,amsmath, amsthm}
\usepackage[accepted]{icml2021}
\usepackage{mathtools}
\newcommand{\R}{\mathbb{R}}
\newcommand{\E}{\mathbb{E}}
\newtheorem{definition}{Definition} 
\newtheorem*{remark}{Remark} 
\newtheorem{lemma}{Lemma} 
\newtheorem{theorem}{Theorem}
\newtheorem*{corollary}{Corollary}  

\usepackage{pgfplots}
\pgfplotsset{compat=1.12}
\definecolor{crimson}{RGB}{220, 20, 60}
\definecolor{indigo}{RGB}{75, 0, 30}
\usepackage{bigints}

\icmltitlerunning{On Characterizing GAN Convergence Through Proximal Duality Gap}
  
\begin{document}

\twocolumn[
\icmltitle{On Characterizing GAN Convergence Through Proximal Duality Gap}

\icmlkeywords{Machine Learning, ICML}
\begin{icmlauthorlist}
\icmlauthor{Sahil Sidheekh}{iit}
\icmlauthor{Aroof Aimen}{iit}
\icmlauthor{Narayanan C. Krishnan}{iit}
\end{icmlauthorlist}

\icmlaffiliation{iit}{Department of Computer Science, Indian Institute of Technology, Ropar, India}
\icmlcorrespondingauthor{Sahil Sidheekh}{2017csb1104@iitrpr.ac.in}
% \icmlcorrespondingauthor{Aroof Aimen}{2018csz0001@iitrpr.ac.in}
\icmlcorrespondingauthor{Narayanan C Krishnan}{ckn@iitrpr.ac.in}

\vskip 0.3in
]

\printAffiliationsAndNotice{}

\begin{abstract}
   Despite the accomplishments of Generative Adversarial Networks (GANs) in modeling data distributions, training them remains a challenging task. A contributing factor to this difficulty is the non-intuitive nature of the GAN loss curves, which necessitates a subjective evaluation of the generated output to infer training progress. Recently, motivated by game theory, duality gap has been proposed as a domain agnostic measure to monitor GAN training. However, it is restricted to the setting when the GAN converges to a Nash equilibrium. But GANs need not always converge to a Nash equilibrium to model the data distribution. In this work, we extend the notion of duality gap to proximal duality gap that is applicable to the general context of training GANs where Nash equilibria may not exist. We show theoretically that the proximal duality gap is capable of monitoring the convergence of GANs to a wider spectrum of equilibria that subsumes Nash equilibria. We also theoretically establish the relationship between the proximal duality gap and the divergence between the real and generated data distributions for different GAN formulations. Our results provide new insights into the nature of GAN convergence. Finally, we validate experimentally the usefulness of proximal duality gap for monitoring and influencing GAN training. 
\end{abstract}

\section{Introduction} \label{section_1}

Generative modeling is an important machine learning paradigm, aiming to learn data distributions. The ability to parametrically model the true underlying distribution of real-world data from a given empirical distribution brings with it the power to generate new and unseen instances. Generative adversarial networks (GANs) are perhaps the most popular and successful of innovations for learning data distributions. A GAN formulates the generative modeling problem as a zero-sum game between two agents - a Discriminator (D) and a Generator (G). The discriminator aims to differentiate the fake samples produced by the generator from samples belonging to the true data distribution. On the other hand, the generator seeks to fool the discriminator by learning a mapping from an input noise space to the data space. The generator can also be viewed as performing adversarial attacks on the discriminator, exploiting the information leak through the discriminator and learning the real data distribution as the game proceeds to an equilibrium. 

Formally, the GAN game is defined as : 
\begin{equation}\label{eq:gan_game}
    \underset{\theta_g\in \Theta_G}{\min} \ \underset{\theta_d \in \Theta_D}{\max} \ V(D_{\theta_d},G_{\theta_g}) ,
\end{equation}
where the generator (parametrized by $\theta_g$) and discriminator (parametrized by $\theta_d$) are neural networks and $V$ is the objective function that the agents seek to optimize. Different GAN formulations yield different expressions for $V$, each minimizing a unique divergence between the real and generated data distributions. The classic GAN formulation \cite{goodfellow2014generative} minimizes the JS divergence and is defined by :
\begin{equation*}
    \centering
      V={\mathop{\mathbb{E}_{ \textbf{x} \sim P_{r}}}[\log(D( \textbf{x}))] + 
      \mathop{\mathbb{E}_{\textbf{x} \sim P_{\theta_g}}}[\log (1-D(\textbf{x}))]}
\end{equation*}
where $P_r$ denotes the real data distribution and $P_{\theta_g}$ denotes the generated data distribution.

In any learning problem, the trajectory of the loss functions should indicate the goodness of the trained model. However, such intuitive inferences cannot be drawn from the loss curves of a GAN. This is because classical training of a GAN involves alternate gradient descent optimization of the objective function w.r.t the individual agents.
% that is expected to converge only on attaining a Nash Equilibrium.
Each optimization step of an agent alters its adversary's loss surface, resulting in non-intuitive loss curves for both the agents over time.
% The alternate optimization results in  changes to the loss surface of each agent at every iteration.
Figure \ref{fig:loss_curves} shows discriminator and generator loss curves for a GAN when it (a) converges and (b) diverges. Ideally, losses should decrease during model convergence and increase during divergence. However, we observe a diminishing generator loss and an increasing discriminator loss when the GAN converges. When it diverges, there is an interplay between both the losses. These loss curves do not give any insight into the gradual improvement or degradation in the GAN's performance.
 \begin{figure}[t]
    % \fbox{\rule{0pt}{2in} \rule{.9\linewidth}{0pt}}
    \centering
    \begin{tabular}{cc}
    \subfloat[Convergence]{\includegraphics[width = 0.48\linewidth]{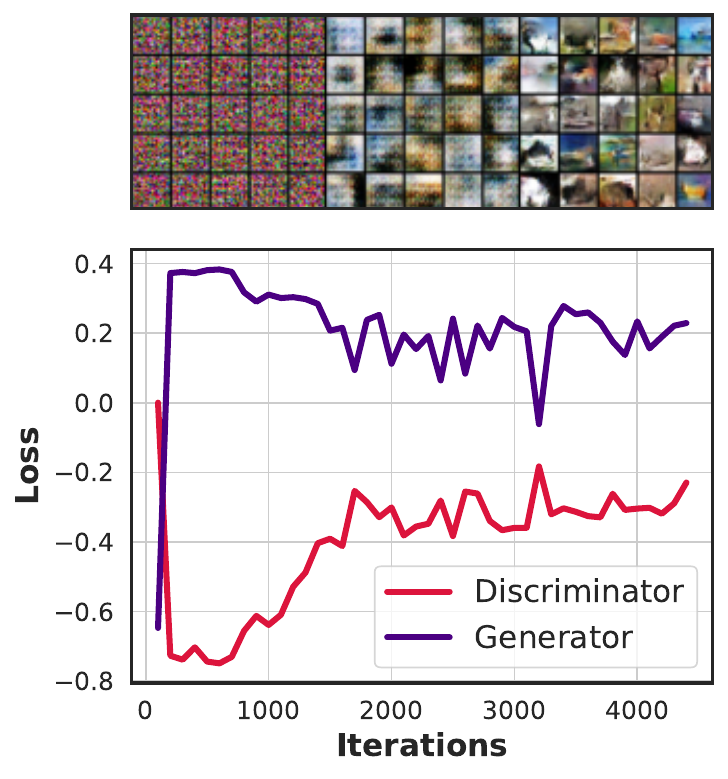}}&
    \subfloat[Divergence]{\includegraphics[width = 0.452\linewidth]{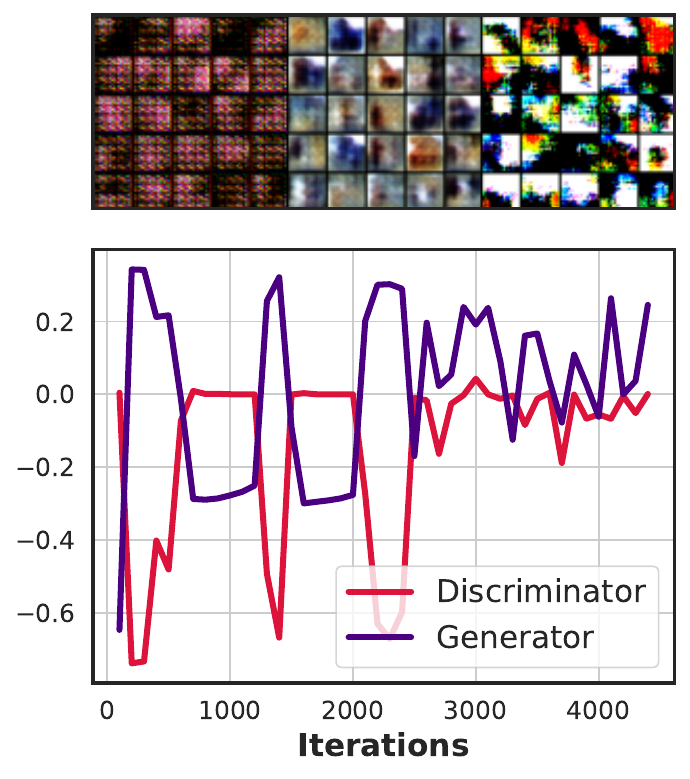}}\\
    \end{tabular}
    \caption{Loss curves throughout the training progress of WGAN on the CIFAR-10 dataset.}
     \label{fig:loss_curves}
\end{figure}
% Classical training of a GAN involves alternate gradient descent optimization of the objective function w.r.t the individual agents that is expected to converge only on attaining a Nash Equilibrium. Each optimization step of an agent alters the adversary's loss surface, resulting in highly non-intuitive loss curves for both the agents over time. 
Thus, monitoring GANs often requires a subjective evaluation of the generated output. As a result, an exhaustive search over the architecture and hyperparameter space to find the delicate balance demanded by the GAN game becomes infeasible. This increases the complexity of GAN training that is already challenging due to the instabilities posed by the min-max gradient optimization. Objective measures capable of quantifying GAN training progress can reduce the training complexity.

 Duality Gap ($DG$) \cite{grnarova2019domain} for GANs, motivated by principles of game theory, is a recently proposed objective measure for monitoring GAN training. The $DG$ quantifies a GAN configuration's goodness in terms of the agents' ability to deviate from it in search of better optima.  When a GAN converges to a Nash equilibrium, no agent can unilaterally deviate to find a better optima, and hence the $DG$ would be zero. The ability to quantify convergence as well as the domain agnostic nature that requires no pre-trained models nor labeled data, makes $DG$ a potentially powerful tool to monitor GAN training.

However, $DG$ relies on the notion that GANs converge to Nash equilibria. On the contrary, recent studies \cite{farnia2020gans,berard2019closer} suggest that a Nash equilibrium need not always exist for a GAN, especially when trained under regularized environments that most modern GAN formulations employ. GANs can converge to stable stationary points that are not Nash equilibria, all the while producing realistic data samples with high fidelity. This weakens the foundation upon which the notion of $DG$ as a performance monitoring tool for GANs is built, eliciting the following questions: Can GANs capture the real data distribution even at non-Nash critical points? If so, would the $DG$ at such stationary points be zero? If not, how do we monitor the GAN training in such situations?

In this work, we study the above questions by introducing the notion of proximal duality gap for GANs that is generalizable to scenarios where Nash equilibria may not exist. Our work is motivated by the notion of proximal equilibria for GANs that serves as a general concept for characterizing GAN optimality \cite{farnia2020gans}. We define the $DG$ in terms of the agents ability to optimize the proximal GAN objective (see Eq. \ref{proximal_objective}) and call it as the Proximal Duality Gap ($DG^{\lambda}$). % We define the proximal duality gap in terms of the agents ability to optimize the proximal objective :
% \begin{equation}
% V^{prox}(D,G) = \underset{\theta_d \in \Theta_{D^'}}{\max} \ V(D^{'},G) - \lambda ||D^{'} - D||^2
% \end{equation}
A proximal equilibrium for the GAN game (Eq. \ref{eq:gan_game}) is a Nash equilibrium w.r.t the proximal objective. Thus, whenever the GAN game attains a proximal equilibrium, $DG^{\lambda}$ will tend to zero, indicating model convergence. As all Nash equilibria form a subset of proximal equilibria, $DG^{\lambda}$ serves as a generic and robust measure that can quantify GAN convergence in the wild.

% Further, we use the notion of $DG^\lambda$ to theoretically understand the nature of \textcolor{red}{real time} GAN convergence. We derive the relationship between the \textcolor{red}{data distribution and the game equilibrium} and show, for various GAN formulations, that the generator learns the real data distribution only at a stackelberg equilibrium. Proximal duality gap also enables us to derive theoretical justifications for the empirical observation that GANs produce realistic data samples even at non-Nash stable points.

Overall, we make the following contributions:
\begin{itemize}
    \item We present an acute limitation of $DG$ for monitoring GAN training.
    \item We propose a theoretically grounded and robust extension - $DG^{\lambda}$, that overcomes this limitation and is also applicable to the broader context, when GANs converge to a non-Nash equilibrium.
    \item Using $DG^{\lambda}$, we derive insights into the nature of GAN convergence. Specifically, we study the relationship between the quality of the learned data distribution and the game equilibria. We show that for various GANs, a configuration $(\theta_d, \theta_g)$ where $P_{\theta_g}=P_r$ corresponds to a Stackelberg equilibrium.
    \item We demonstrate through experiments, the proficiency of $DG^{\lambda}$ for monitoring and influencing GAN training.
\end{itemize}

\section{Related Work}
Motivated by the non-inferrable nature of GAN loss curves, developing extrinsic measures to monitor GAN training has emerged as an active research area \cite{DBLP:journals/corr/abs-1802-03446, lucic2018gans,DBLP:journals/corr/abs-1808-04888}. Existing measures such as average log-likelihood \cite{goodfellow2014generative, theis2015note}, Inception Score (IS) \cite{salimans2016improved}, Frechet Inception Distance (FID) \cite{heusel2017gans} evaluate the output of GANs, but require pre-trained models. Further, these measures, including the more recent ones such as precision and recall \cite{sajjadi2018assessing, kynkaanniemi2019improved}, density and coverage \cite{tolstikhin2017adagan,DBLP:journals/corr/abs-2002-09797} do not monitor the training progress nor characterize the equilibria of the GAN game.

Duality Gap ($DG$) \cite{grnarova2019domain} is a recently proposed domain agnostic and computationally feasible metric for monitoring and evaluating GAN training. $DG$ is zero when the GAN converges to a Nash equilibrium making it an attractive metric for objectively monitoring GAN training. However, $DG$ has a fundamental limitation with its estimation process due to vanishing gradients. Adding perturbations to the GAN configuration before estimating the duality gap (perturbed duality gap) helps to overcome this issue \cite{sidheekh2020duality}. But both these approaches assume the convergence of GANs to Nash equilibria, which may not always be the case, especially for high dimensional datasets (as we demonstrate in the next section). Thus, limiting the applicability of $DG$ and perturbed $DG$ for monitoring GAN training.

%Recent literature suggests that GANs need not converge to a Nash equilibrium and may converge to a Proximal equilibrium \cite{farnia2020gans}. We propose a metric - Proximal Duality Gap ($DG^{\lambda}$) that monitors the training progress of GANs by estimating its capability to optimize the proximal GAN objective.

\section{Background}
\subsection{A Brief Overview of GAN formulations}
% \subsubsection{Classic GAN}
\textbf{Classic GAN :}
    The min-max objective in the classic GAN \cite{goodfellow2014generative} formulation is :
    \begin{equation}
        \label{classic_gan_objective}
        V_c = \dfrac{1}{2} \mathop{\mathbb{E}_{ \textbf{x} \sim P_{r}}}[\log D(\textbf{x})]   + \dfrac{1}{2} \mathop{\mathbb{E}_{\textbf{x} \sim P_{{\theta_g}}}}[\log (1-D(\textbf{x}))]
    \end{equation}

 where the probabilistic discriminator $D$ outputs the likelihood of the input data point belonging to the real data distribution. The discriminator's objective is to maximize the log-likelihood to learn the conditional probability $P(y|\textbf{x})$, where $y=0 \text{ and }1$ indicate a fake and real data point respectively. Minimizing the above objective w.r.t the generator for the optimal discriminator is equivalent to minimizing the Jenson Shannon divergence (JSD) between $P_{\theta_g}$ and $P_r$.
 
% \subsubsection{F-GAN}
\textbf{F-GAN:}
F-GAN \cite{nowozin2016f} is the generalization of the classic GAN to minimize arbitrary $f-$ divergences by incorporating an extension of the variational divergence estimation framework \cite{nguyen}. For a convex, lower semi-continuous function $f:\R_+ \rightarrow \R$ that satisfies $f(1)$ = 0, the $f-$ divergence between distributions $P$, and $Q$ is,
%F-GAN \cite{nowozin2016f} is the generalization of the classic GAN to minimize arbitrary $f$-divergences by incorporating an extension of the variational divergence estimation framework \cite{nguyen} within the GAN framework, enabling the generalization from JS divergence minimization of classic GAN to arbitrary $f$-divergences. For a convex, lower semi-continuous function $f:\R_+ \rightarrow \R$ that satisfies $f(1)$ = 0, we can define an associated $f-$divergence,
\begin{equation}
    \label{f_gan_div}
    D_f(P||Q)=\int p(x)f\left(\dfrac{q(x)}{p(x)}\right)dx
\end{equation}
The F-GAN objective that minimizes the $f-$ divergence ($D_f$) between $P_{\theta_g}$  and $P_r$ is defined as
\begin{equation}
    \label{f_gan_objective}
    V_f=\mathop{\mathbb{E}_{ \textbf{x} \sim P_{r}}}[D(\textbf{x})]   - \mathop{\mathbb{E}_{\textbf{x} \sim P_{{\theta_g}}}}[f^*(D(\textbf{x}))]
\end{equation}
where  $f^*(x)=\underset{t\in \text{Dom}(f)}{\sup}\{xt-f(t)\}$ is the Fenchel-conjugate of $f$. The classic GAN is a special case of F-GAN when $ f(t)=t\log(t) - (t+1)\log\left (\dfrac{t+1}{2}\right)$.

% \subsubsection{Wasserstein GAN (WGAN)}
\textbf{Wasserstein GAN (WGAN):}
WGAN formulates the GAN game as a minimization of the optimal transport cost, the Wasserstein distance, between $P_r$ and $P_{\theta_g}$, a more efficient cost function to learn data distributions having support on low dimensional manifolds \cite{arjovsky-chintala-bottou-2017}.  Specifically, the Kantorovich-Rubinstein duality is used to arrive at the Wasserstein-1 (Earth Movers) distance between the distributions defined as: $ \underset{||D||_{L \leq 1}}{\sup} \E_{\textbf{x} \sim P_{r}}[D(\textbf{x})] - \E_{\textbf{x} \sim P_{{\theta_g}}}[D(\textbf{x})]$, where the supremum is over all 1-Lipschitz discriminators. In practice, the Lipschitz constraint is enforced through weight clipping resulting in the following WGAN objective:
\begin{equation}
V_{w_1} = \E_{\textbf{x} \sim P_{r}}[D(\textbf{x})] - \E_{\textbf{x} \sim P_{{\theta_g}}}[D(\textbf{x})]
\end{equation}
The WGAN formulation is also extended to a general transport cost \cite{convexduality_farnia} $c(\textbf{x},\textbf{y})$ by constraining the discriminators to be $c$-concave as
\begin{equation}
        \label{wgan_objective}
        V_w = \E_{\textbf{x} \sim {P}_{r}}[D(\textbf{x})] - \mathbb{E}_{\textbf{x} \sim {P}_{\theta_g}}[D^c(\textbf{x})],
\end{equation}
where $D^c$ is the c-transform of the discriminator $D$, i.e.,
\begin{equation}
D^c(\textbf{x}) = \underset{\textbf{y}}{\sup} \ \{ D(\textbf{y}) - c(\textbf{x,y}) \}
\end{equation}
and the Wasserstein distance between $P_{\theta_g}$ and $P_r$ is :
\begin{equation}
    W_c(P_{\theta_g}||P_r) = \underset{D-c-concave}{\sup} \E_{\textbf{x} \sim P_{r}}[D(\textbf{x})] - \E_{\textbf{x} \sim P_{{\theta_g}}}[D^c(\textbf{x})]
\end{equation}
Despite the added stability of Wasserstein distance over other divergences, training GANs remains an arduous task. This has motivated efforts towards understanding the nature of GAN convergence.

\subsection{Understanding GAN convergence}
% \subsubsection{Classical Notion of Equilibrium}
\textbf{Classical Notion of GAN Equilibrium}

 Traditionally, the GAN game was expected to converge to a pure Nash equilibrium, a configuration ($\theta_d^*,\theta_g^*$) that is optimal for both the players i.e.
 \begin{equation*}
\centering
\underset{\theta_d}{\max}\ V(D_{\theta_d},G_{\theta_g^*})  =  \underset{\theta_g}{\min}\ V(D_{\theta_d^*},G_{\theta_g}) = V(D_{\theta_d^*},G_{\theta_g^*})
\end{equation*}
%  \begin{equation}
%  \label{pure_Nash_eq}
%  \underset{\theta_d}{\max}\ V(D_{\theta_d},G_{\theta_g^*}) = \underset{\theta_g}{\min}\ V(D_{\theta_d^*},G_{\theta_g}) = V(D_{\theta_d^*},G_{\theta_g^*}), \forall \theta_g, \theta_d    
%  \end{equation}
 A GAN having unbounded capacity learns the true data distribution at such a solution \cite{goodfellow2014generative}. However, a pure Nash equilibrium need not always exist for a zero sum game  \cite{Nash48}. Only an extended notion - the mixed strategy Nash equilibrium (MNE) is guaranteed to exist. Recent GAN formulations explicitly seek the MNE \cite{arora17ageneralization,hsieh2019finding}. As a mixed strategy gives a distribution over the model parameters, the stationary point to which the GAN converges need not be individually optimal for both the players. %Duality gap, thus, would not be an ideal candidate for monitoring convergence to MNEs.

\textbf{GANs Need Not Converge to Nash Equilibria}
%  \begin{figure*}[t]
%     % \fbox{\rule{0pt}{2in} \rule{.9\linewidth}{0pt}}
%     \centering
%     \begin{tabular}{ccc}
%     \subfloat{\includegraphics[width = 0.2\linewidth]{figures/non_nash_demo/m2_sngan_mnist.png}}&
%     \subfloat{\includegraphics[width = 0.2\linewidth]{figures/non_nash_demo/m2_ciafar_sngan.png}}&
%     \subfloat{\includegraphics[width = 0.2\linewidth]{figures/non_nash_demo/m2_sngan_celeba.png}}\\
%     \subfloat{\includegraphics[width = 0.2\linewidth]{figures/non_nash_demo/sngan_cifar_eigv.eps}}&
%     % \fbox{\rule{0pt}{2in} \rule{.9\linewidth}{0pt}}
%     \subfloat{\includegraphics[width = 0.2\linewidth]{figures/non_nash_demo/sngan_cifar_eigv.eps}}&
%     \subfloat{\includegraphics[width = 0.2\linewidth]{figures/non_nash_demo/sngan_cifar_eigv.eps}} 
%     \end{tabular}
%     \caption{GANs converged to non-Nash attractors, yet outputting high fidelity images.}
%     \label{fig:non_nash_demo}
% \end{figure*}

 GAN convergence has also been well studied from an optimization perspective using the notion of stability \cite{daskalakis2017training,fiez2019convergence, DBLP:conf/nips/NouiehedSHLR19,DBLP:conf/nips/ZhangYB19,mazumdar2019onfinding,mokhtari2020unified,lin2020gradient}. As GAN formulations are non-convex, only local surrogates of equilibria may be attainable while employing gradient based optimization. A (differential) local Nash equilibrium (LNE) ($\theta_d^*,\theta_g^*$) satisfies two properties: 1) $\nabla_{\theta_d}V(\theta_d^*,\theta_g^*) = \nabla_{\theta_g}V(\theta_d^*,\theta_g^*) =  0$ and 2) $\nabla^2_{\theta_d}V(\theta_d^*,\theta_g^*) \prec 0 \ , \  \nabla^2_{\theta_g}V(\theta_d^*,\theta_g^*) \succ 0$.
%  \begin{equation}
%      \nabla_{\theta_d}V(\theta_d^*,\theta_g^*) = \nabla_{\theta_g}V(\theta_d^*,\theta_g^*) =  0
%  \end{equation}
%  \begin{equation}
%     \label{nash_sufficient_cond}
%      \nabla^2_{\theta_d}V(\theta_d^*,\theta_g^*) \prec 0 \ , \  \nabla^2_{\theta_g}V(\theta_d^*,\theta_g^*) \succ 0
%  \end{equation}
 However, recent literature \cite{DBLP:mazumdar2018onthe,adolphs19localSaddle} suggests the existence of many stable attractors that are not LNE for the alternate gradient descent optimization, nor produce realistic data and proposes methods to escape these attractors.  \cite{fiez2019convergence,jin2020local} study the gradient dynamics of sequential games and establish that their only stable attractors are Stackelberg equilibria. However, the relationship between the learned data distribution and the game configuration has not been well studied. Further, recent empirical studies \cite{berard2019closer,farnia2020gans} also suggest that \textit{GANs need not attain an LNE to produce realistic data}.
 
 \begin{figure}[h]
    % \fbox{\rule{0pt}{2in} \rule{.9\linewidth}{0pt}}
    \centering
    \includegraphics[width = 0.56\linewidth]{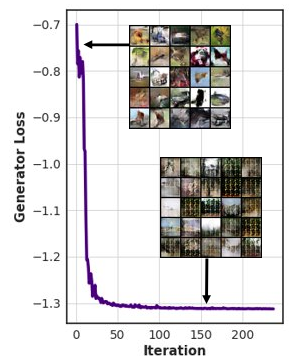}
    \includegraphics[width = 0.32\linewidth]{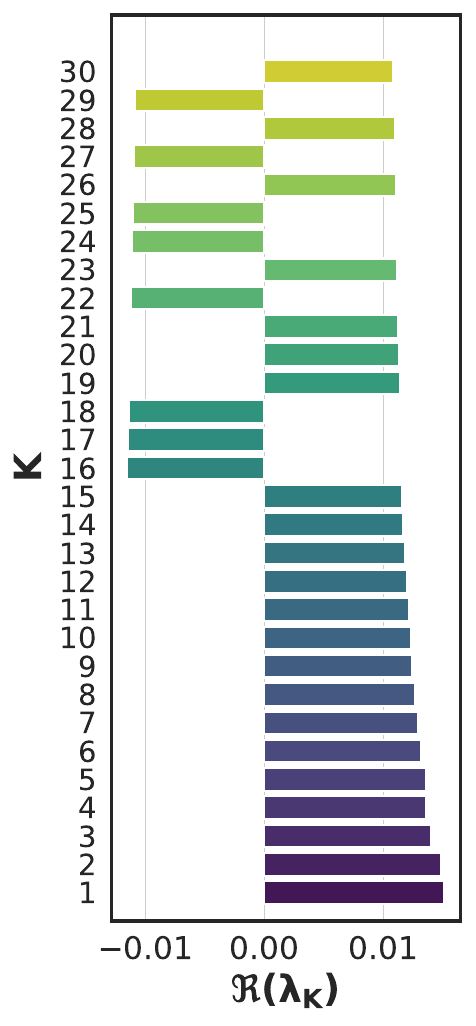}
    % \caption{GAN converged to non-Nash attractor, yet outputting high fidelity images.}
    \caption{The high fidelity images  outputted by a converged GAN deteriorates on optimizing only w.r.t the generator while attaining a lower loss (left), indicating that the GAN has not converged to a Nash equilibrium; confirmed by the positive and negative eigenvalues of the Hessian (right).}
    \label{fig:non_nash_demo}
\end{figure}
 We verify the above hypothesis by training a spectral normalized GAN (SNGAN) on the CIFAR-10 dataset for 100 epochs ensuring that the models have converged producing high fidelity samples. As a local Nash equilibrium is locally optimal for the agents, optimizing the objective function w.r.t any individual player should not facilitate departure of the player from such a point. However, as demonstrated in Figure \ref{fig:non_nash_demo}, the generator deviates from the attained equilibrium on further optimizing the objective function w.r.t only the generator. While the generator attains a lower cost, there is a clear deterioration in the generated samples' quality. This implies that the GAN has not converged to an LNE. We further strengthen the claim by verifying the top-K eigenvalues ($\lambda_K$) (by magnitude) of the Hessian of the objective w.r.t the generator's parameters. However, as shown in Figure \ref{fig:non_nash_demo}, the Hessian has both positive as well as negative eigen values violating the second property of an LNE, thus emphasizing that GANs can converge to non-Nash attractors, all the while producing high fidelity samples. Similar results on GANs trained for MNIST and CELEB-A datasets are discussed in the supplementary material. %In this work we introduce a measure that can help quantify GAN convergence and understand the nature of the generated distribution even under such scenarios. 
 
% \begin{equation}
%      V(\theta_d,\theta_g^*) \leqslant V(\theta_d^*,\theta_g^*) \leqslant V(\theta_d* ,\theta_g) 
% \end{equation}

% \subsubsection{Proximal Equilibria for GANs}
\textbf{Proximal Equilibria for GANs}

 Majority of studies attempting to characterize the equilibria for GANs assume unbounded capacity for the models (realizable setting) \cite{goodfellow2014generative,arora17ageneralization,hsieh2019finding}. However, in practice, most GAN architectures  \cite{brock2018large,zhang2019self,miyato2018spectral,gulrajani2017improved,arjovsky-chintala-bottou-2017,DBLP:journals/corr/RadfordMC15} employ normalization and regularization to achieve state of the art performance. %The theoretical results that hold for an equilibria in a realizable setting may fail under such constraints.
 A recent study \cite{farnia2020gans} on GAN convergence under the non-realizable setting proposes a more generic notion of equilibira - the Proximal Equilibria (PE). This notion of equilibira is derived from a sequential game-play perspective for a GAN, for which a Stackelberg equilibrium is guaranteed to exist under mild continuity assumptions \cite{jin2020local,fiez2019convergence}.
% Formally, a Stackelberg equilibrium ($\theta_d^*,\theta_g^*$) satifies 
%  $ \theta_g^* = \underset{\theta_g \in \Theta_G}{argmin} \{ \underset{\theta_d \in \Theta_D}{\max} V(D_{\theta_d},G_{\theta_g})\} $ and $\theta_d^* = \underset{\theta_d \in \Theta_D}{argmax} V(D_{\theta_d},G_{\theta_g^*})$
 Farnia and Ozdaglar define a proximal operator over the original GAN objective, allowing the discriminator to be optimal in a neighbourhood (controlled by $\lambda$),
\begin{equation}
\label{proximal_objective}
V^{\lambda}(D_{\theta_d},G_{\theta_g}) = \underset{\Tilde{\theta_d} \in \Theta_{D}}{\max} \ V(D_{\Tilde{\theta_d}},G_{\theta_g}) - \lambda ||D_{\Tilde{\theta_d}} - D_{\theta_d}||^2
\end{equation}

%  Formally, a Stackelberg equilibrium ($\theta_d^*,\theta_g^*$) satifies 
%  $ \theta_g^* = \underset{\theta_g \in \Theta_G}{argmin} \{ \underset{\theta_d \in \Theta_D}{\max} V(D_{\theta_d},G_{\theta_g})\} $ and $
%      \theta_d^* = \underset{\theta_d \in \Theta_D}{argmax} V(D_{\theta_d},G_{\theta_g^*})
%  $. The proximal equilibrium is a restricted stackelberg equilibrium where the generator is optimal w.r.t all discriminators only in a $\lambda-$neighbourhood of the current discriminator.
 A proximal equilibrium is defined as the Nash equilibrium for the objective $V^{\lambda}$, which is guaranteed to exist \cite{farnia2020gans}.
%  , as it forms a Stackelberg equilibrium for the original objective $V$.
 Formally, a configuration $(\theta_d^*,\theta_g^*)$ of the GAN game (Eq \ref{eq:gan_game}) is called a $\lambda-$proximal equilibrium if and only if $\forall  \ \theta_d, \theta_g$,
%  \begin{align}
%      V(D_{\theta_d},G_{\theta_g^*}) \leq V(D_{\theta_d^*},G_{\theta_g^*}) \\ \leq  \underset{\Tilde{\theta_d} \in \Theta_{D}}{\max} \ V(D_{\Tilde{\theta_d}},G_{\theta_g}) - \lambda ||D_{\Tilde{\theta_d}} - D_{\theta_d^*}||^2
%  \end{align}
\begin{equation}
\begin{aligned}
V(D_{\theta_d},G_{\theta_g^*}) &\leq{} V(D_{\theta_d^*},G_{\theta_g^*}) \\
      & \leq  \underset{\Tilde{\theta_d} \in \Theta_{D}}{\max} \ V(D_{\Tilde{\theta_d}},G_{\theta_g}) - \lambda ||D_{\Tilde{\theta_d}} - D_{\theta_d^*}||^2 \\
\end{aligned}
\label{eq:pe}
\end{equation}
 As the extreme cases $\lambda \rightarrow \infty (0)$ recreate a Nash (Stackelberg) equilibrium, the  $\lambda-$proximal equilibrium explores the spectrum of equilibria between the two and thus serves as a generic notion for GAN convergence. \cite{farnia2020gans} also suggest proximal training to explicitly enforce convergence to a proximal equilibrium. Our work, in contrast, is aimed towards evaluating GAN convergence and quantifying its goodness, irrespective of how it was trained.

\section{Proximal Duality Gap}
We first define the classical duality gap \cite{grnarova2019domain} for GANs before moving on to the proposed measure.
\begin{definition}
\label{DG_definition}
Consider the GAN game presented in Eq.\ref{eq:gan_game}. Then, for a configuration ($\theta_d,\theta_g$) of the game, the duality gap ($DG$) is defined as :

$DG(\theta_d,\theta_g)=  \underset{\Tilde{\theta_d} \in \Theta_D}{\max} \ V(D_{ \Tilde{\theta_d}},G_{\theta_g}) - \underset{\Tilde{\theta_g} \in \Theta_G}{\min} \ V(D_{\theta_d},G_{\Tilde{\theta_g}})$  
\end{definition}
At a pure Nash equilibrium ($\theta_d^*,\theta_g^*$), $DG (\theta_d^*,\theta_g^*)=0$. 
 $DG$ has some interesting properties. First, it is lower bounded by the JSD between $P_r$ and $P_{\theta_g}$ (for $V$ = $V_c$). Second, as $DG$ is applicable to any GAN objective $V(D,G)$ and does not require pre-trained classifier nor labeled data, it is domain agnostic and potentially better equipped to monitor GAN training over other prevalent evaluation measures. However, as discussed previously, GANs can converge to non Nash attractors, where $P_r$ and $P_{\theta_g}$ are aligned well. $DG$ at such  equilibria is not very well understood, limiting its practicality for monitoring GAN training.

We extend the notion of Duality Gap to the general context of training GANs where Nash equilibria need not be attainable, utilizing the proximal operator. We define the Proximal Duality Gap ($DG^\lambda$) as below.
\begin{definition}
 \label{PDG_definition}
The proximal duality gap ($DG^{\lambda}$) at $(\theta_d,\theta_g)$ for the GAN game presented in Eq.\ref{eq:gan_game} is defined as 
%Consider the GAN game presented in Eq.\ref{eq:gan_game}. Then, for a configuration $(\theta_d,\theta_g)$ of the game, the $\lambda-$proximal duality gap ($DG^{\lambda}$) is defined by :
\begin{eqnarray*}
DG^{\lambda}(\theta_d,\theta_g) & = & V_{D_{w}}(\theta_g) - V_{G_{w}}^{\lambda}(\theta_d), \ \text{where} \\
V_{D_{w}}(\theta_g) & = & \underset{\theta_d^{'} \in \Theta_D}{\max} \ V(D_{ \theta_d^{'}},G_{\theta_g}) \\ 
V_{G_{w}}^{\lambda}(\theta_d) & = & \underset{\theta_g^{'} \in \Theta_G}{\min} \ V^{\lambda}(D_{\theta_d},G_{\theta_g^{'}})
\end{eqnarray*}
% where,
% V^{prox}_{\lambda}(D_{\theta_d},G_{\theta_g}) = \underset{\Tilde{\theta_d} \in \Theta_{D}}{\max} \ V(D_{\Tilde{\theta_d}},G_{\theta_g}) - \lambda ||D_{\Tilde{\theta_d}} - D_{\theta_d}||^2
\end{definition}
The terms $D_{w}$( or $G_{w}$) indicate the worst adversary that the generator (or discriminator) might face. Note that $\max_{\theta_d^{'} \in \Theta_D}V^{\lambda}(D_{\theta_d^{'}},G_{\theta_g}) = \max_{\theta_d^{'} \in \Theta_D}\ V(D_{\theta_d^{'}},G_{\theta_g})$. Thus, for a GAN configuration attained using $V$, the $DG^{\lambda}$ measures the ability of the agents to deviate from it w.r.t the proximal objective $V^{\lambda}$.
% \begin{lemma}
% $\max_{\theta_d^{'} \in \Theta_D}V^{\lambda}(D_{\theta_d^{'}},G_{\theta_g}) = \max_{\theta_d^{'} \in \Theta_D}\ V(D_{\theta_d^{'}},G_{\theta_g})$
% \end{lemma}
%Thus, for a GAN configuration attained using an objective $V$, the proximal duality gap measures the ability of the agents to deviate from it w.r.t the proximal objective $V^{\lambda}$.
% Since proximal equilibria correspond to Nash equilibria of the proximal objective, no agent can unilaterally achieve a better payoff by optimizing the proximal objective at such a configuration.
% Thus, when the GAN converges to a proximal equilibrium, that need not necessarily be a Nash equilibrium w.r.t $V$, the proximal duality gap tends to zero.
\begin{remark}
For all proximal equilibria ($\theta_d^*,\theta_g^*$) of the GAN game defined by Eq. \ref{eq:gan_game}, $V_{D_{w}}(\theta_g^*)=V^{\lambda}_{G_{w}}(\theta_d^*)=V^{\lambda}(\theta_d^*,\theta_g^*)$, thus  $DG^{\lambda}(\theta_d^*,\theta_g^*)=0$.
\end{remark}
This remark directly follows from the definition of proximal equilibria (Eq. \ref{eq:pe}). Thus $DG^{\lambda}$ tending to zero implies that the GAN game has converged to a proximal equilibrium. However, as GANs are used for learning data distributions, it is important that measures to quantify GAN convergence should also give insights into the nature of the learned data distribution. Thus, to establish the applicability of $DG^{\lambda}$, we study how it relates to the divergence between the real and generated data distributions for various GAN formulations.
% As proximal equilibria serve as a generic optimality notion for GANs, proximal duality gap overcomes the limitations of duality gap and can be used to monitor GAN training in the wild.
% \subsubsection{M_1^{$\lambda$}}
\subsection{Theoretical Analysis}
The definition of $DG^\lambda$ has two terms - $V_{D_w}$ and $V^\lambda_{G_w}$. We first establish the relationship between $V_{D_w}$ and the divergences ($DIV$) used in various GAN formulations - the JS divergence ($JSD(P_{\theta_g}||P_r)$) for classical GAN objective $V_c$, the Wasserstein distance ($W_c(P_{\theta_g}||P_r)$) for the WGAN objective $V_w$, and the $f-$divergence ($D_f(P_{\theta_g}||P_r)$) for the F-GAN objective $V_f$. 
 In our analysis, following \cite{farnia2020gans,grnarova2019domain,arjovsky-chintala-bottou-2017}, we assume that for a fixed generator, an optimal discriminator ($D_w$) that maximizes $V$ exists.

\begin{lemma}
Given a generator $\theta_g$, $V_{D_w}$ is related to the divergences between $P_r$ and $P_{\theta_g}$ in the various GAN objectives as follows
\begin{equation*}
    V_{D_w}(\theta_g)=
    \begin{cases}
      JSD(P_{\theta_g}||P_r) - \log 2, & \text{if}\ V=V_c \\
      W_c(P_{\theta_g}||P_r), & \text{if}\ V=V_w \\
      D_f(P_{\theta_g}||P_r), & \text{if}\ V=V_f 
    \end{cases}
\end{equation*}
\end{lemma}
\begin{proof}
Deferred to the supplementary material.
\end{proof}
Thus, $V_{D_w}(\theta_g)$ measures the quality of the generator $G_{\theta_g}$. If $\theta_g$ is optimal such that $P_{\theta_g}$ covers the real distribution $P_r$, then $V_{D_w}(\theta_g)$ achieves the minimum value. In case of a mismatch between $P_{\theta_g}$ and $P_r$, either due to insufficient support or poor sample quality, $V_{D_w}(\theta_g)$ will increase, thus making it a potentially useful metric in itself to monitor GAN training. However, $V_{D_w}$ does not incorporate the ability of the discriminator to deviate from the current game configuration. Thus, it cannot identify if the game has converged to an equilibrium. Further, as the minimum value for $V_{D_{w}}$ will vary depending upon the GAN formulation, it cannot serve as a domain agnostic measure for monitoring GAN training. $DG^{\lambda}$ on the other hand will always tend to zero on attaining a proximal equilibrium irrespective of the GAN formulation. Thus, in the next result we analyze the behavior of $DG^\lambda$ for the three GAN formulations. 
Specifically, we show that both in the realizable setting ($G$ is of unbounded capacity; $\exists \  \theta_g$ such that $P_{\theta_g}=P_r$) and the more practical  non-realizable setting (where $G$ is of bounded capacity), the proximal duality gap is positive and lower bounded closely by the divergence between $P_r$ and $P_{\theta_g}$.
\begin{theorem}
Consider a GAN game governed by an objective function $V$. Then the proximal duality gap ($DG^\lambda$) at a configuration ($\theta_d,\theta_g$) is related to the  divergence between the real ($P_r$) and generated ($P_{\theta_g}$) data distributions as follows.

\begin{equation*}
    DG^{\lambda}(\theta_d,\theta_g) \geq DIV(P_{\theta_g}||P_r) - \kappa
\end{equation*}
where,
\begin{equation*}
    DIV(P_{\theta_g}||P_r)=
    \begin{cases}
      JSD(P_{\theta_g}||P_r), & \text{if}\ V=V_c \\
      W_c(P_{\theta_g}||P_r), & \text{if}\ V=V_w \\
      D_f(P_{\theta_g}||P_r), & \text{if}\ V=V_f \\
    \end{cases}
\end{equation*}
 and $\kappa \ (\geq0)$ denotes the minimum divergence that the considered class of generator functions can achieve with the real data distribution.
\end{theorem}
\begin{proof}
Deferred to the supplementary material.
\end{proof}
\begin{corollary}
Under the realizable setting, since $\exists \  \theta_g$ such that $P_{\theta_g}=P_r$, $\kappa=0$ and hence $DG^{\lambda}(\theta_d,\theta_g) \geq DIV(P_{\theta_g}||P_r)$. 
\end{corollary}

This theorem non-trivially extends the prior result on the bound of $DG$ only for classic GAN \cite{grnarova2019domain} under the realizable setting. 
\subsection{Implications}
As $DG^{\lambda}$ is lower bounded by the divergence between the real and generated data distributions, $DG^{\lambda} \rightarrow 0 $ not only implies that the GAN has reached an equilibrium, but also that generated distribution is close to the real data distribution. 
\begin{corollary}
For GAN formulations defined by $V_c, V_w$ or $V_f$, the generator attains the minimum possible divergence with the real data distribution at a proximal equilibrium ($\theta_d^*,\theta_g^*$), as $DG^{\lambda}(\theta_d^*,\theta_g^*) = 0$. 
\end{corollary}
This furthers the adeptness of proximal equilibria to serve as a general optimality notion for GANs. As a proximal equilibrium need not be a Nash equilibrium, it also facilitates the following interesting observation:
\begin{remark}
GANs can capture the real data distribution even at non Nash game configurations. 
\end{remark}
Thus, the empirical observation \cite{berard2019closer} that GANs can produce realistic data samples having high fidelity despite converging to a non-Nash attractor of the gradient dynamics is theoretically justified. 

On similar lines, a natural question that arises concerning the behaviour of $DG^{\lambda}$ is whether proximal equilibria constitute an exhaustive notion of equilibria at which GANs can capture the real data distribution. Precisely, we ask the question: Does GAN converging to a solution such that $P_{\theta_g} \rightarrow P_{r} \implies DG^{\lambda} \rightarrow 0?$ Our answer begins with the following proposition concerning the extreme case for $DG^{\lambda}$ as $ \lambda \rightarrow 0$.

\begin{theorem}
\label{thm:pdg_zero_stackelberg}
The proximal duality gap ($DG^{\lambda}$) at a configuration  ($\theta_d^*,\theta_g^*$) for the  GAN game defined by $V_c, V_w,$ or $V_f$ is equal to zero for $\lambda=0$, when the generator learns the real data distribution; $P_{\theta_g^*}=P_r \implies DG^{\lambda=0}(\theta_d^*,\theta_g^*) = 0$.
% Let ($\theta_d^*,\theta_g^*$) denote a GAN equilibrium such that $P_{\theta_g^*}=P_r$. Then, for $\lambda = 0$, $DG^{\lambda}(\theta_d^*,\theta_g^*) = 0$
\end{theorem}
\begin{proof}
Deferred to the supplementary material.
\end{proof}
\begin{corollary}
For the GAN formulations defined by $V_c, V_w$ or $V_f$, the generator learns the real data distribution at a configuration ($\theta_d^*,\theta_g^*$) if and only if ($\theta_d^*,\theta_g^*$) constitutes a Stackelberg equilibrium.
\end{corollary}
\begin{proof}
Deferred to the supplementary material.
\end{proof}
Thus $DG^{\lambda=0}(\theta_d,\theta_g) \rightarrow 0$ whenever $P_{\theta_g} \rightarrow P_r$. The value of $\lambda$ restricts the discriminator in the proximal objective ($V^{\lambda}$) to be optimal within a neighbourhood. As $\lambda \rightarrow 0$, $V^{\lambda}$ considers the optimal discriminator over the entire parameter space ($\Theta_D$). Thus $DG^{\lambda=0}(\theta_d,\theta_g) = 0$ implies that ($\theta_d,\theta_g$) is a Stackelberg equilibrium. 
% The implication of Theorem \ref{thm:pdg_zero_stackelberg} is that a GAN configuration ($\theta_d^*,\theta_g^*$) that captures the real data distribution is always a stackelberg equilibrium.
As all proximal equilibria form a subset of Stackelberg equilibria, $DG^{\lambda=0}$ would thus be an ideal choice to monitor GAN convergence. However, as $\lambda$ decreases, the complexity of computing $V^{\lambda}$ increases rapidly and becomes infeasible as $\lambda \rightarrow 0$. Hence, it is only practical to check if a GAN configuration is a $\lambda(>0)-$proximal equilibrium. But can $DG^{\lambda}$, for a fixed value of $\lambda(> 0)$ monitor convergence of GANs to all $\lambda^{'}-$proximal equilibria? To address this question, let us study the two cases - (i) $\lambda^{'} \geq \lambda$ and (ii) $\lambda^{'} < \lambda$ separately. The following theorem addresses case (i) utilizing the hierarchical property of proximal equilibria.
\begin{theorem}
\label{thm:hierarchical_pdg}
Consider a GAN configuration $(\theta_d,\theta_g)$. Then,  $\forall \lambda^{'} \geq \lambda_{0}$,
\begin{equation*}
DG^{\lambda=\lambda^{'}}(\theta_d,\theta_g) = 0 \implies DG^{\lambda=\lambda_{0}}(\theta_d,\theta_g) = 0 
\end{equation*}
\end{theorem}
\begin{proof}
Deferred to Supplementary material
\end{proof}
$DG^{\lambda^{'}}(\theta_d,\theta_g) = 0$ is a sufficient condition for ($\theta_d,\theta_g$) being a $\lambda^{'}-$proximal equilibrium. Thus, it follows from theorem \ref{thm:hierarchical_pdg} that $DG^{\lambda}$ is adept to monitor convergence of GANs to all $\lambda^{'} (\geq \lambda)-$proximal equilibria. However, when a GAN converges to a  $\lambda^{'}(< \lambda)-$proximal equilibrium, $DG^{\lambda}$ can be prone to error. The following theorem addresses this issue by upper bounding the difference between $DG^{\lambda}$ and the divergence between real and generated data distributions.

\begin{theorem}
\label{thm:pdg_upperbound}
Consider a GAN game governed by an objective function $V$. For $\lambda > 0$, let $V^{\lambda}$ denote the proximal objective defined by $V^{\lambda}(\theta_d,\theta_g)=max_{\Tilde{\theta_d}}V(\Tilde{\theta_d},\theta_g)-\lambda||D_{\Tilde{\theta_d}}-D_{\theta_d}||^2$ . Then, $\forall \ \epsilon > 0, \ \exists \ \delta > 0$ such that if $||D_{\theta_d}-D_{\Tilde{\theta_d}}||<\delta$, then $DG^{\lambda}(\theta_d,\theta_g) - DIV(P_{\theta_g}||P_r)<\epsilon$
% \begin{equation*}
%     ||D_{\theta_d}-D_{\Tilde{\theta_d}}||<\delta \implies DG^{\lambda}(\theta_d,\theta_g) - DIV(P_r||P_{\theta_g})<\epsilon
% \end{equation*}
where,
\begin{equation*}
    DIV(P_{\theta_g}||P_r)=
    \begin{cases}
      JSD(P_{\theta_g}||P_r), & \text{if}\ V=V_c \\
      W_c(P_{\theta_g}||P_r), & \text{if}\ V=V_w \\
      D_f(P_{\theta_g}||P_r), & \text{if}\ V=V_f \\
    \end{cases}
\end{equation*}
\end{theorem}

\begin{proof}
Deferred to Supplementary material
\end{proof}
\begin{corollary}
\label{cor:pdg_epsilon}
For a GAN configuration ($\theta_d^*,\theta_g^*$) such that $P_{\theta_g^*}=P_r, DG^{\lambda}(\theta_d^*,\theta_g^*) < \epsilon$
\end{corollary}
Thus even when the GAN converges to a $\lambda^{'}(<\lambda)-$proximal equilibrium, the error that $DG^{\lambda}$ can incur is bounded. Previously, the absence of such an upper bound as implied by theorem \ref{thm:pdg_upperbound} for $DG$ meant that $DG$ need not necessarily be close to zero when $P_{\theta_g}$ is close to $P_r$. $DG^{\lambda}$ however, rules out this possibility and is thus a theoretically grounded and robust measure that can serve as a tool for monitoring convergence of GANs in the wild.

\begin{figure*}[t!]
    \centering
    \begin{tikzpicture}
        \begin{axis}[%
        hide axis,
        legend columns=2,
        xmin=10,
        xmax=50,
        ymin=0,
        ymax=0.4,
        legend style={draw=white!15!black,legend cell align=right}
        ]
        \hspace{10pt}
        \addlegendimage{crimson,line width=2. pt}
        \addlegendentry{\small
 Proximal Duality Gap ($DG^{\lambda}$) };
        \addlegendimage{indigo,line width=2. pt,mark options={solid}}
        \addlegendentry{\small
Classical Duality Gap ($DG$)};
        \end{axis}
    \end{tikzpicture}\\
    \begin{tikzpicture}
        \node [draw=black!2,rotate=90,anchor=center,fill=gray!8] { \hspace{0.082\linewidth} \textbf{WGAN}\hspace{0.09\linewidth} };    
    \end{tikzpicture}
    \includegraphics[width = 0.26\linewidth]{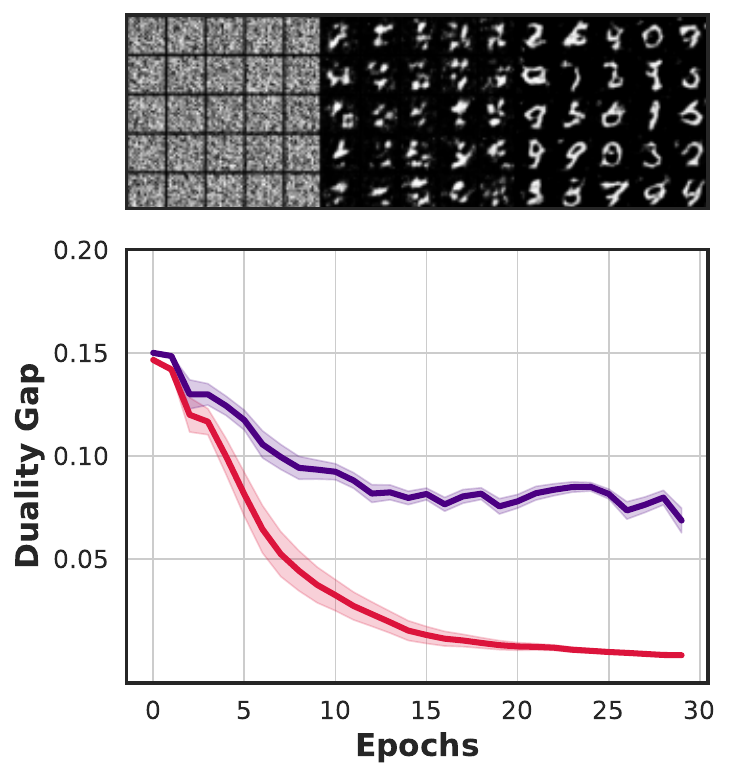}
    \includegraphics[width = 0.26\linewidth]{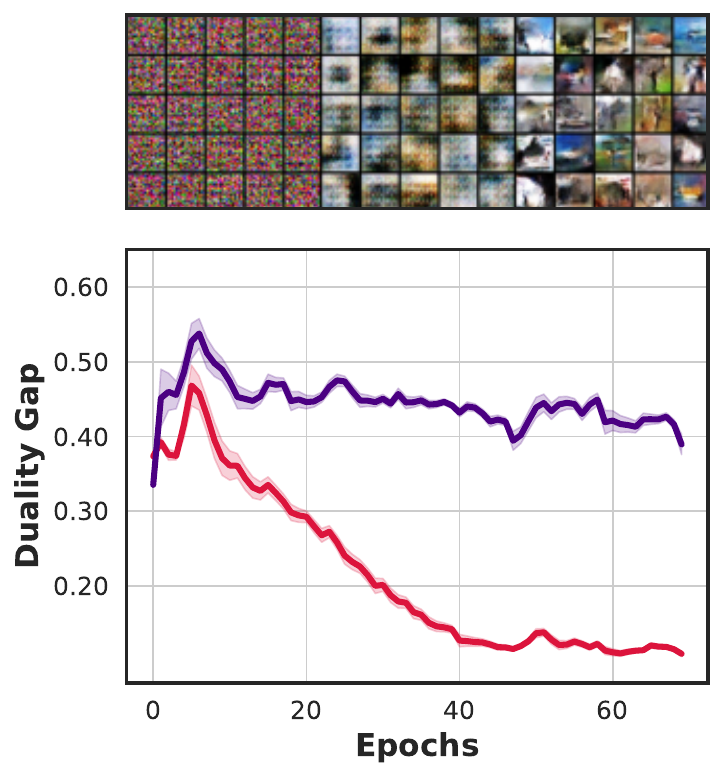}
    \includegraphics[width = 0.26\linewidth]{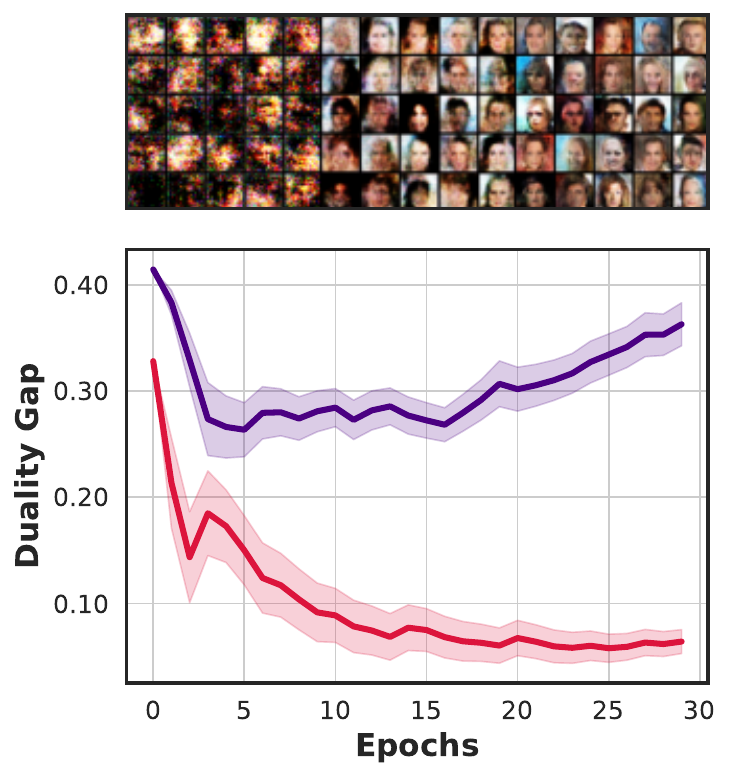}\\
    \begin{tikzpicture}
        \node [draw=black!2,rotate=90,anchor=center,fill=gray!8] { \hspace{0.082\linewidth} \textbf{SNGAN} \hspace{0.09\linewidth} };    
    \end{tikzpicture}
    \includegraphics[width = 0.26\linewidth]{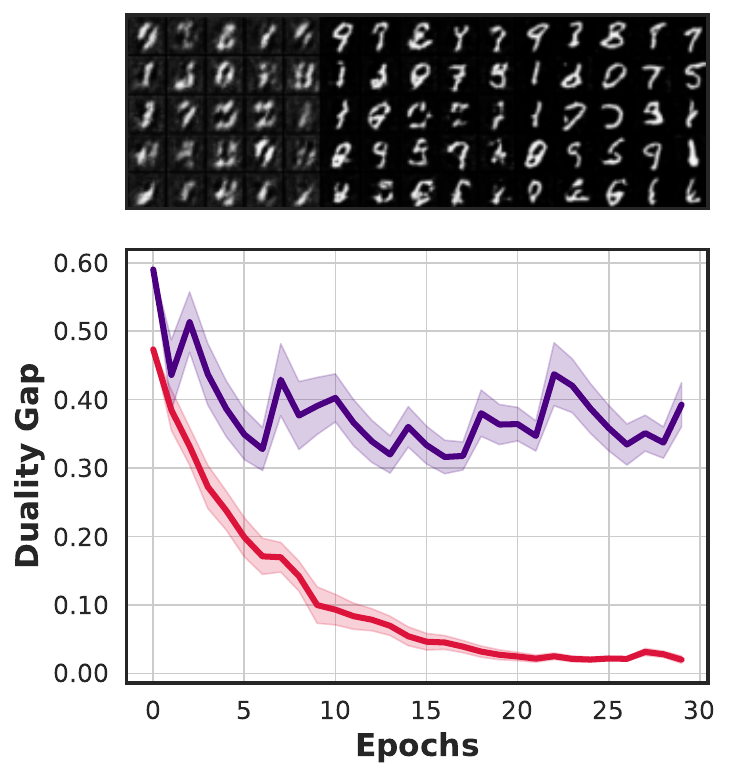}
    \includegraphics[width = 0.26\linewidth]{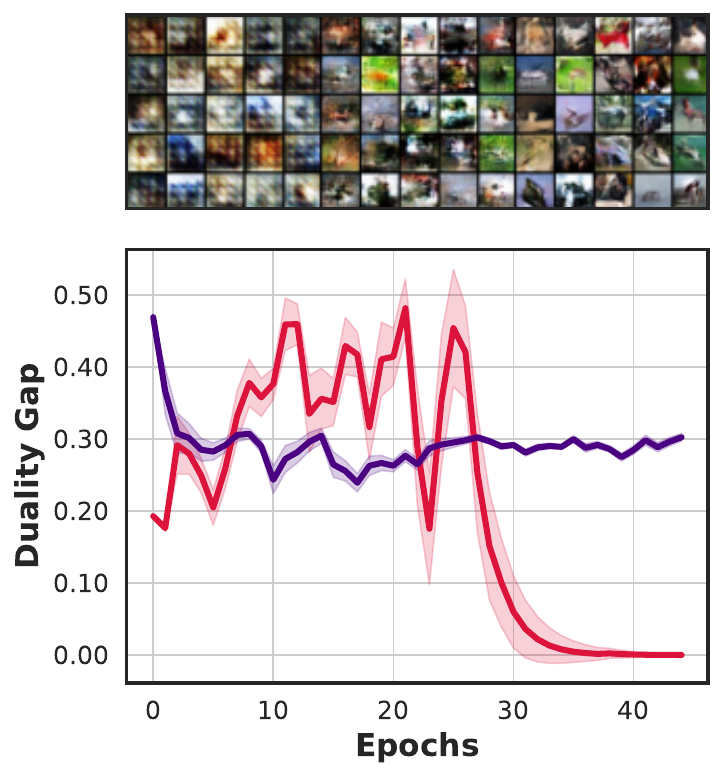}
    \includegraphics[width = 0.26\linewidth]{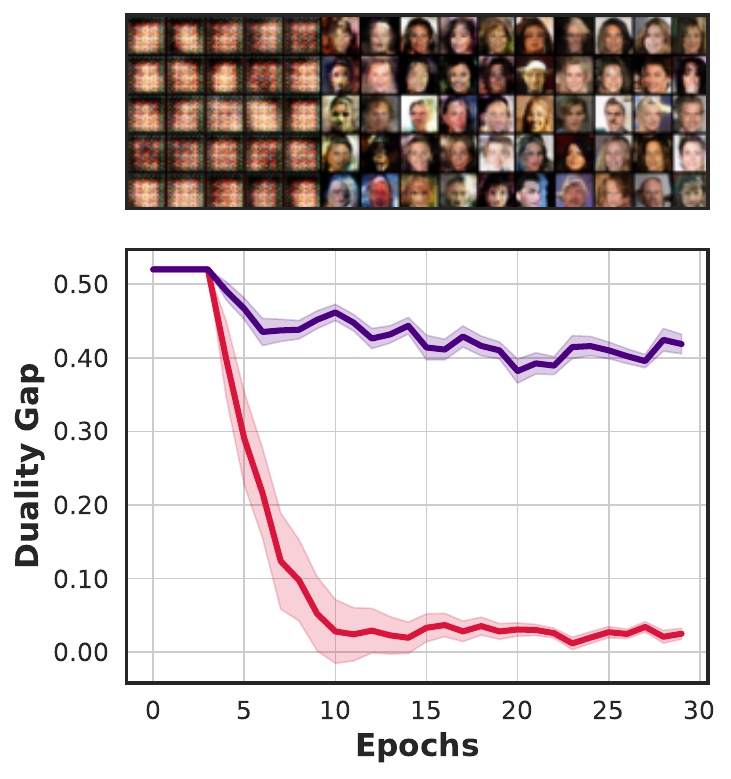}
    \caption{Monitoring Convergence of WGAN (Top Row) and SNGAN (Bottom Row) over the datasets MNIST (Col 1), CIFAR-10 (Col 2) and CELEB-A (Col 3) using duality gap. $DG$ is not reflective of the GAN convergence. $DG^{\lambda}$, on the other hand is indicative of convergence and saturates close to zero. The shaded region indicates the standard deviation over 5 independent trials.}
    \label{fig:monitoring_gan_convergence}
\end{figure*}
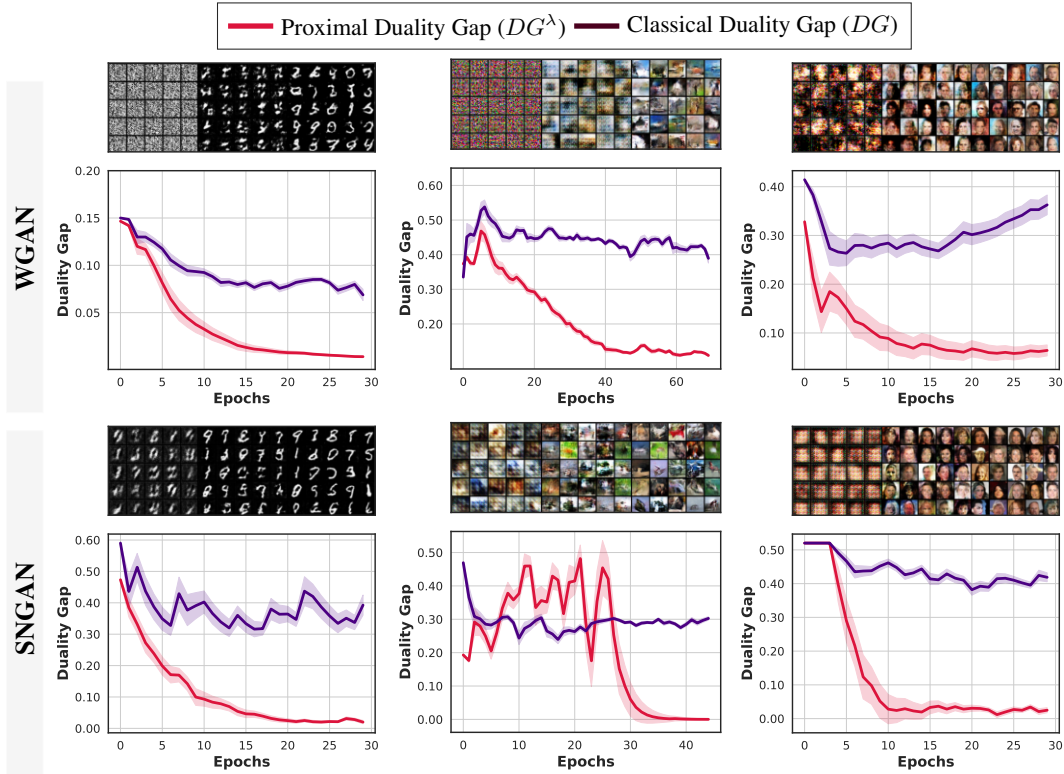

\subsection{Estimating Proximal Duality Gap}
Computing the true $DG^{\lambda}$ for a GAN configuration is a hard task as it involves finding the optima of non convex functions. We approximate $DG^{\lambda}$ by employing gradient descent to estimate $V_{D_{w}}$ and $ V_{G_{w}}^{\lambda}$.  Following \cite{farnia2020gans}, we use the Sobolev norm in the proximal objective $V^{\lambda}$, given by  $||D|| = \sqrt{\E_{x \sim P_r}[||\nabla_{x}D(x)||^{2}_{2}]}$. For a GAN game governed by $V$, estimating $V_{D_{w}}$ involves optimizing $V(\theta_d,\theta_g)$ w.r.t $\theta_d$ using gradient descent. However, estimating $V_{G_{w}}^{\lambda}$ requires gradient computation over the proximal operator. 
As shown in \cite{farnia2020gans}, for a GAN objective function $V$ that is smooth w.r.t $\theta_d$, the gradient of the proximal objective ($V^{\lambda}$) w.r.t $\theta_g$ can be obtained in terms of $V$ as : $\nabla_{\theta_g} V^{\lambda} (D_{\theta_d}, G_{\theta_g}) = \nabla_{\theta_g} V (D_{\theta_d^{*}},G_{\theta_g})$, where $\theta_d^*$ represents the optimal discriminator implied by the proximal objective. Thus, to estimate $V^{\lambda}_{G_w}$, at every iteration we use gradient descent to obtain $\theta_d^*$ for the corresponding $\theta_g$ and update $\theta_g$ to minimize $V(\theta_d^*,\theta_g)$. The algorithm  for the overall estimation process and the associated computational complexity are discussed in the supplementary material. To ensure that we obtain an unbiased estimate for $DG^{\lambda}$, following \cite{grnarova2019domain}, we split the dataset into 3 disjoint sets - $S_A$, $S_B$ and $S_C$. We train the GAN using $S_A$, we use $S_B$ to find the worst case counter parts $D_w$ and $G_w$ via gradient descent, and $S_C$ to evaluate the objective function at the obtained worst case configurations.
\section{Experimentation}
To experimentally establish the proficiency of $DG^{\lambda}$, we consider a WGAN with weight-clipping (that optimizes $V_w$) \cite{arjovsky-chintala-bottou-2017} and a Spectral Normalized GAN (SNGAN) (that optimizes $V_c$) \cite{miyato2018spectral} over 3 datasets - MNIST \cite{deng2012the}, CIFAR-10 \cite{cite-cifar10} and CELEB-A \cite{liu2015faceattributes}. For all the experiments, we use the 4-layer DCGAN \cite{DBLP:journals/corr/RadfordMC15} architecture for both the generator and the discriminator networks, and an Adam optimizer \cite{DBLP:journals/corr/KingmaB14} to train the models. To compute $DG^{\lambda}$, we use $\lambda$=0.1 and 20 optimization steps for approximating the proximal objective. We used the torchgan framework \cite{pal2019torchgan} to train and evaluate all GAN models.  Further implementation details for each experiment are provided in the supplementary material and the source code is publicly available
\footnote{\url{https://github.com/proximal-dg/proximal_dg}}.

\subsection*{Monitoring GAN training using DG{$^\lambda$}}

Our first experiment aims to establish that $DG^{\lambda}$ is better equipped over $DG$ to monitor GAN convergence in practice. To this end, we train a WGAN and SNGAN over the 3 datasets till the models have converged. We compute $DG$ and $DG^{\lambda}$ throughout the training process, in addition to the (image) domain specific evaluation measures - IS and FID. Figure \ref{fig:monitoring_gan_convergence} demonstrates the training progress of each GAN, qualitatively through visualization of samples from the learned data distribution and quantitatively in terms of $DG$ and $DG^{\lambda}$. The high fidelity of the learned data samples indicates that the models have converged. However, $DG$ is not reflective of the training progress. This suggests that the GANs have not attained a Nash equilibrium - the behaviour of $DG$ at non-Nash critical points is not well understood. $DG^{\lambda}$, on the other hand, captures the trend in the training progress, and eventually saturates close to zero. Thus, it is able to better characterize convergence. We also quantitatively validate the above observation by examining the correlation between $DG$ and $DG^{\lambda}$ against popular measures - IS and FID that quantify the quality of $P_{\theta_g}$. As shown in Table \ref{table:correlation}, $DG^{\lambda}$ has a higher positive correlation with FID and a higher negative correlation with IS as compared to $DG$. The duality gap is negatively correlated with IS because the latter increases as the GAN learns the real data distribution, whereas the former decreases. A larger (or smaller) IS (or FID) implies better fidelity of the learned data distribution. The higher correlation of $DG^{\lambda}$ with IS and FID over the training process thus validates that $DG^{\lambda}$ is adept to monitor not only the convergence of GANs to an equilibrium but also the goodness of $P_{\theta_g}$.

\begin{table}[t]
\centering
\begin{tabular}{@{}lcccc@{}}
\toprule
\textbf{}         & \multicolumn{4}{c}{\textbf{Pearson Correlation Coefficient (r)}}            \\ 
                  & r$_{DG,IS}$ & r$_{DG^{\lambda},IS}$ & r$_{DG,FID}$ & r$_{DG^{\lambda},FID}$ \\ \midrule
\textit{MNIST}    & $-0.752$      & $\mathbf{-0.892}$                & $0.845$        & $\mathbf{0.957}$                  \\
\textit{CIFAR-10} & $-0.368$      & $\mathbf{-0.854}$                & $0.213$        & $\mathbf{0.738}$                 \\
\textit{CELEB-A}  & $-0.213$      & $\mathbf{-0.524}$                & $0.263$        & $\mathbf{0.699}$                  \\ \bottomrule
\end{tabular}
\caption{Comparing the correlation of $DG$ and $DG^{\lambda}$ with IS and FID computed during the training of WGAN over the 3 datasets.}
\label{table:correlation}
\end{table}

\subsection*{Visualizing the effect of $\lambda$}

\begin{figure}[h]
    \centering
    \includegraphics[width =0.75\linewidth]{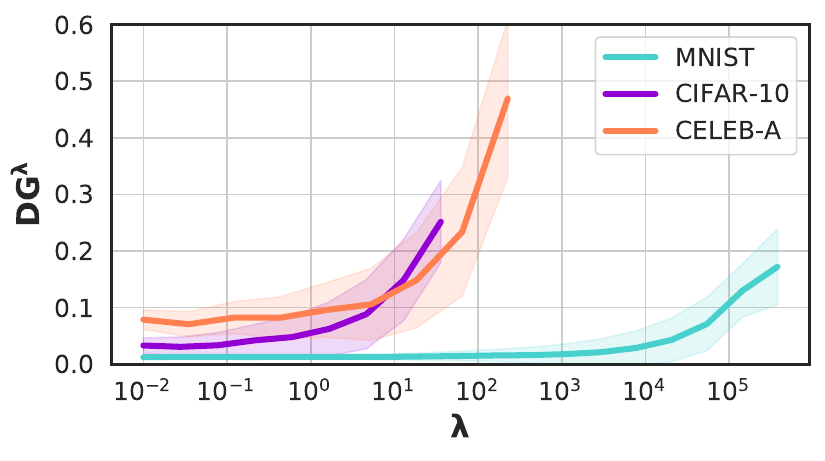}
    \caption{ The behaviour of $DG^{\lambda}$ for increasing $\lambda$}
     \label{fig:ablation_lamda}
\end{figure}

$\lambda$ is a critical hyperparameter that determines the proficiency of $DG^{\lambda}$. We observed from the theoretical analysis that, while $DG^{\lambda}$ is adept to monitor convergence of GANs to all $\lambda^{'}(\geq\lambda)$-proximal equilibria, it is prone to error as $\lambda$ increases. As $\lambda \rightarrow \infty$, $DG^{\lambda}$ becomes equivalent to $DG$. We thus experimentally study the behaviour of $DG^{\lambda}$ for increasing values of $\lambda$. We compute $DG^{\lambda}$ at the converged WGAN configurations (as shown in Fig \ref{fig:monitoring_gan_convergence}) for each of the three datasets by varying $\lambda$ in the range $[10^{-2},10^{6}]$. We observe (Figure \ref{fig:ablation_lamda}) that for all the datasets, $DG^{\lambda}$ remains close to zero and unaffected for  $\lambda$ in range $[10^{-2},1]$. Interestingly, for the MNIST dataset, $DG^{\lambda}$ remains unaffected even for larger values of $\lambda(\approx 10^4)$. This suggests that the WGAN configuration for MNIST is closer to a Nash equilibrium, also explaining why $DG$ and $DG^{\lambda}$ are closer for the same in Figure \ref{fig:monitoring_gan_convergence}. As $\lambda$ crosses a threshold ($10^{1}$ for CELEB-A, CIFAR-10 and $10^4$ for MNIST), $DG^{\lambda}$ increases sharply and behaves similar to $DG$. Thus, for a small value for $\lambda$ ($<1$) $DG^{\lambda}$ is a robust tool for monitoring GAN convergence.

\subsection*{Influencing GAN training using DG$^{\lambda}$}
\begin{figure}[h]
    \centering
    \includegraphics[width =0.85\linewidth]{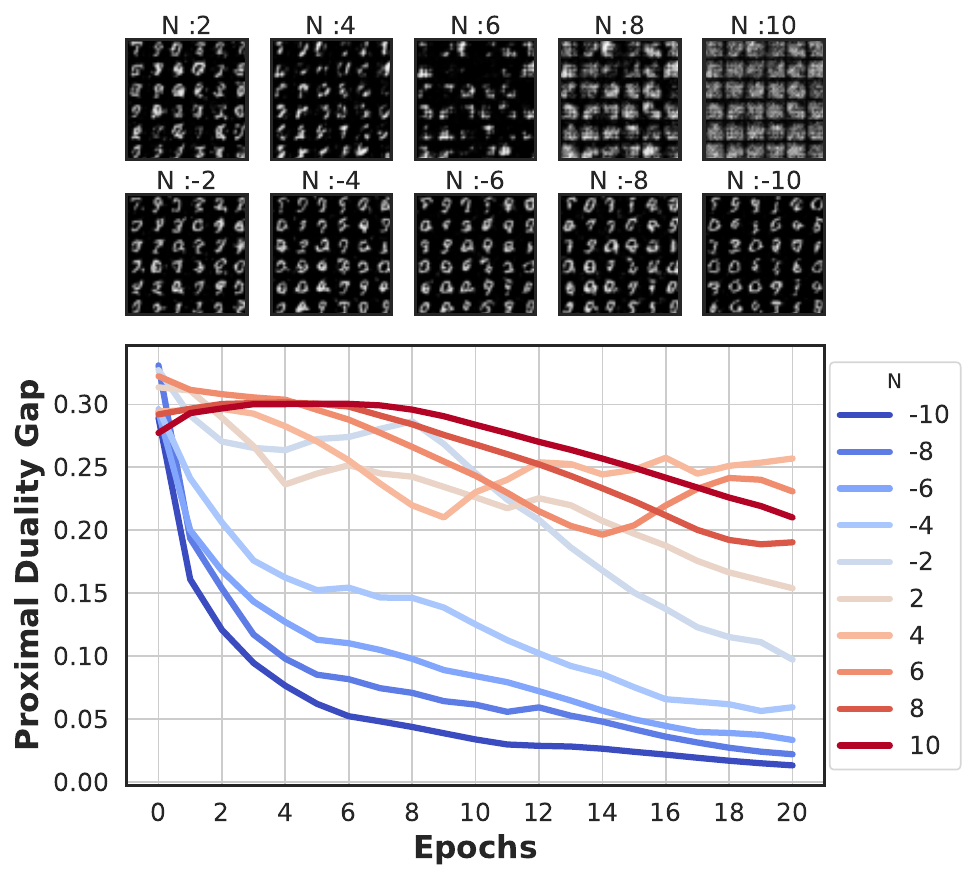}
    \caption{Tuning GAN Hyperparameter using $DG^{\lambda}$}
     \label{fig:hp_tune}
\end{figure}
A quantitative measure to monitor GAN training would enable easier tuning of hyperparameters. In this section, we explore $DG^{\lambda}$ as an effective tool for influencing GAN training. A decisive hyperparameter that governs the delicate balance of the GAN game and hence its convergence is the update ratio of the agents. Let us denote by $N$, the number of discriminator updates per generator update, where a negative value for $N$ indicates a larger number of generator updates. We train WGAN over the MNIST dataset by performing a grid search over $N$ in the range $-10$ to $10$ and computing $DG^{\lambda}$. Figure \ref{fig:hp_tune} depicts the qualitative output at the end of 20 epochs and $DG^{\lambda}$ across training for each value of $N$. We observe that as $N$ increases, the quality of the generated data samples diminishes and the learned data distribution eventually diverges as $N\rightarrow10$. Correspondingly, we observe that $DG^{\lambda}$ is close to zero for lower values of $N$(colored blue) and increases with $N$(colored red),  suggesting that the GAN diverges for larger values of $N$. $DG^{\lambda}$ thus enables us to quantitatively identify the optimal range of values for the hyperparameters of a GAN.

\section{Summary and Future Work }
GANs have pushed the boundaries of learning complex data distributions. However, the non-intuitive nature of GAN loss curves makes training a challenging task. We propose Proximal Duality Gap ($DG^{\lambda}$) as a generic and quantitative tool to monitor GAN training and understand GAN convergence. $DG^{\lambda}$ characterizes GAN convergence as the game attaining a $\lambda-$proximal equilibrium. It also helps derive insights into the nature of GAN convergence - a GAN learns the real data distribution if and only if it attains a Stackelberg equilibrium. The ability of $DG^{\lambda}$ to objectively quantify GAN convergence makes it a useful measure to tune the hyperparameters of a GAN. A couple of open questions that can improve the utility of $DG^{\lambda}$, if addressed, include identifying an optimal $\lambda$ and making the range of values for $DG^{\lambda}$ invariant across GAN formulations. The characterization of GAN convergence through proximal duality gap opens up new avenues for effortless GAN training.
% We believe the proximal duality gap opens up promising avenues for better characterization of GAN convergence and effortless GAN training.

\section*{Acknowledgements}
The resources provided by the \textit{PARAM Shivay} Facility under the National Super-computing Mission, Government of India at the Indian Institute of Technology, Varanasi and under Google Tensorflow Research award are gratefully acknowledged.

\bibliographystyle{icml2021}
\bibliography{output}
% \bibliography{bibmain}

\end{document}

\twocolumn[
\icmltitle{Supplementary Material}

\icmlkeywords{Machine Learning, ICML}

\vskip 0.3in
]

% \section*{Overview}
This document presents a detailed discussion on the theorems, proofs, experimental observations and  setup left out in the main paper due to space constraints. 
% The organization of the supplementary material is summarized below:
% \begin{enumerate}
%     \item \textbf{Background} 
%     \begin{enumerate}
%         \item  \textit{Preliminaries and Notations:}\\
%                 We first introduce the notations to be followed.
%         \item  \textit{GANs Need Not Converge to Nash Equilibrium:}\\ Results supporting the empirical observation that GANs can produce realistic data even at non-Nash critical points.
%     \end{enumerate}
    
%     \item \textbf{Theorems and Proofs}
%     \begin{enumerate}
%             \item \textit{Classical Duality Gap:}\\
%             We formulate the limitation of classical duality gap ($DG$) in monitoring GAN training
%             \item \textit{Proximal Duality Gap:}\\
%              We present the proofs for the results concerning proximal duality gap ($DG^{\lambda}$)
%     \end{enumerate}
%     \item \textbf{$\mathbf{DG^{\lambda}}$ Estimation} \\
%     We discuss the estimation process for $DG^{\lambda}$ and the computational complexity.
%     \item \textbf{Experiments and Results }\\
%     \begin{enumerate}
%         \item \textit{Monitoring GAN Training Using $DG^{\lambda}$}\\
%         Detailed observations and discussion on $DG^{\lambda}$ left out in the main paper. 
%         \item \textit{Implementation and Hyperparameter Details}
%     \end{enumerate}
    
% \end{enumerate}

\section{Background}
\subsection{Preliminaries and Notations}
Consider the GAN game defined by : 
\begin{equation}\label{eq:gan_game}
    \underset{\theta_g\in \Theta_G}{\min} \ \underset{\theta_d \in \Theta_D}{\max} \ V(D_{\theta_d},G_{\theta_g}) ,
\end{equation}
where the generator ($G$, parametrized by $\theta_g$) and discriminator ($D$, parametrized by $\theta_d$) are neural networks and $V$ is the objective function that the agents seek to optimize. Let us denote by $P_r$ the real data distribution and by $P_{\theta_g}$ the generated data distribution. We consider three different GAN formulations, each expressing the objective function ($V$) as summarized below:

\textit{Classic GAN}, defined by:
\begin{equation}
        \label{classic_gan_objective}
        V_c =  \dfrac{1}{2}\mathop{\mathbb{E}_{ \textbf{x} \sim P_{r}}}[\log D(\textbf{x})]   + \dfrac{1}{2}\mathop{\mathbb{E}_{\textbf{x} \sim P_{{\theta_g}}}}[\log (1-D(\textbf{x}))]
\end{equation}
\textit{F-GAN}, defined by:
\begin{equation}
    \label{f_gan_objective}
    V_f=\mathop{\mathbb{E}_{ \textbf{x} \sim P_{r}}}[D(\textbf{x})]   - \mathop{\mathbb{E}_{\textbf{x} \sim P_{{\theta_g}}}}[f^*(D(\textbf{x}))],
\end{equation} where $f^*$ denotes the Fenchel conjugate of a convex lower semi-continuous function $f$ satisfying $f(1)=0$.

\textit{WGAN}, defined by:
\begin{equation}
        \label{wgan_objective}
        V_w = \E_{\textbf{x} \sim {P}_{r}}[D(\textbf{x})] - \mathbb{E}_{\textbf{x} \sim {P}_{\theta_g}}[D^c(\textbf{x})],
\end{equation}
where $D^{c}$ denotes the $c-$transform of $D$.

We denote by $DIV(P_{\theta_g}||P_r)$ the divergence between the real and generated data distributions, defined for the three GAN formulations as below:
\begin{equation*}
    DIV(P_{\theta_g}||P_r)=
    \begin{cases}
      JSD(P_{\theta_g}||P_r), & \text{if}\ V=V_c \\
      W_c(P_{\theta_g}||P_r), & \text{if}\ V=V_w \\
      D_f(P_{\theta_g}||P_r), & \text{if}\ V=V_f \\
    \end{cases}
\end{equation*}
where $JSD$, $W_c$ and $D_f$ denotes the Jenson-Shannon Divergence, Wasserstein Distance (associated with a transport cost $c$) and $f-$divergence respectively. To theoretically study the properties of $DG^{\lambda}$, we assume that for a fixed generator, an optimal discriminator ($D_w$) that maximizes $V$ exists. 

% \begin{align*}
%     JSD(P_{\theta_g}||P_r)&=\dfrac{1}{2}KL( P_{\theta_g}||\dfrac{P_{\theta_g}+P_r}{2}) \\ &  \hspace{50} +\dfrac{1}{2}KL( P_r||\dfrac{P_{\theta_g}+P_r}{2})\\  &\text{ where }  KL(P||Q)=\int P(x)log\left(\dfrac{P(x)}{Q(x)}\right)dx\\
%     W_c(P_{\theta_g}||P_r) &= \underset{D c-concave}{\sup} \E_{x \sim P_{r}}[D(\textbf{x})] - \E_{x \sim P_{{\theta_g}}}[D^c(x)] \\
%     D_f(P_{\theta_g}||P_r)&=\int P_{\theta_g}(x)f\left(\dfrac{P_r(x)}{P_{\theta_g}(x)}\right)dx
% \end{align*}

\begin{figure*}[h]
    % \fbox{\rule{0pt}{2in} \rule{.9\linewidth}{0pt}}
    \centering
    \begin{tabular}{cc}
    \subfloat[SNGAN trained on MNIST]{\includegraphics[width = 0.33\linewidth]{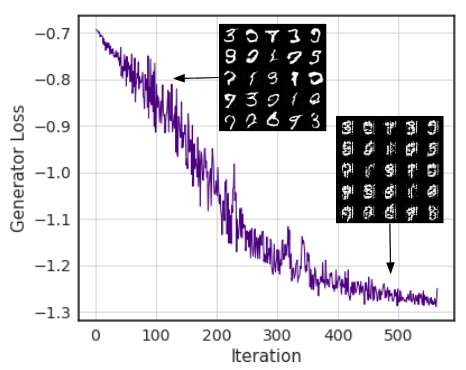}\includegraphics[width = 0.15\linewidth]{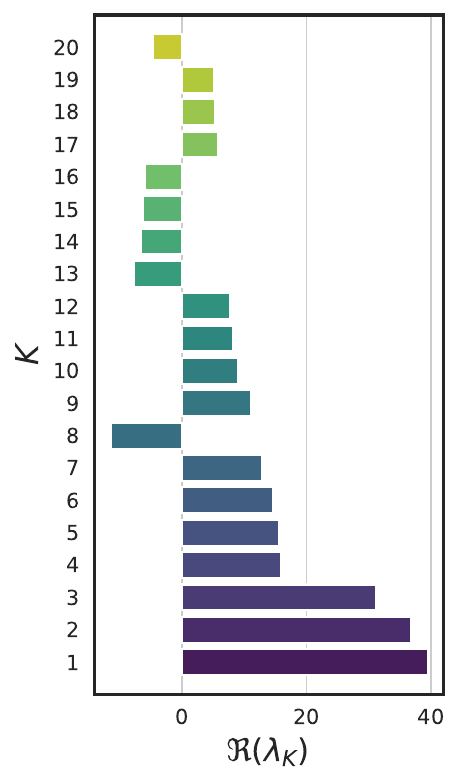}}&
    % \subfloat{\includegraphics[width = 0.3\linewidth]{figures/non_nash_demo/m2_ciafar_sngan.png}}&
    \subfloat[SNGAN trained on CELEB-A]{\includegraphics[width = 0.33\linewidth]{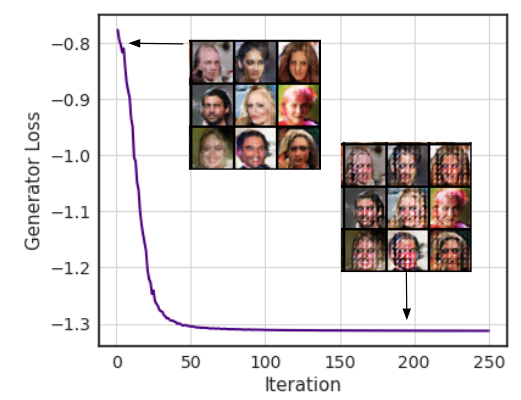}\includegraphics[width = 0.15\linewidth]{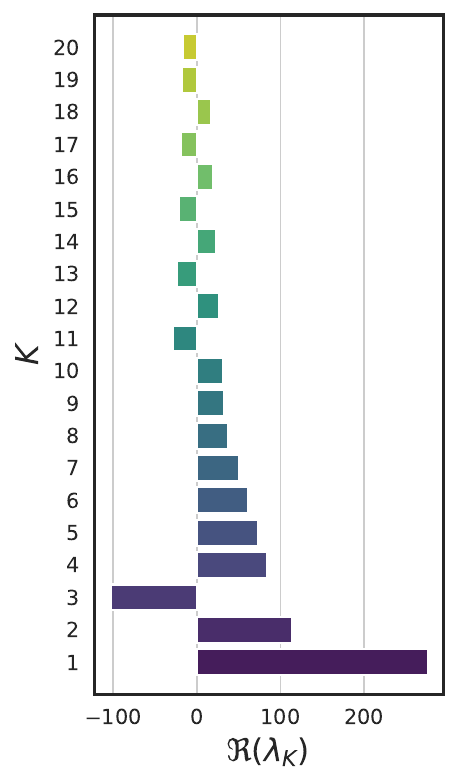}}
    % \subfloat{\includegraphics[width = 0.3\linewidth]{figures/non_nash_demo/sngan_cifar_eigv.eps}}&
    % \fbox{\rule{0pt}{2in} \rule{.9\linewidth}{0pt}}
    
    % \subfloat{\includegraphics[width = 0.35\linewidth]{figures/non_nash_demo/sngan_cifar_eigv.eps}} 
    \end{tabular}
    \caption{The high fidelity images  outputted by a converged GAN deteriorates on optimizing only w.r.t the generator while attaining a lower loss, indicating that the GAN has not converged to a Nash equilibrium; confirmed by the presence of positive and negative eigenvalues in the Hessian.}
    \label{fig:non_nash_demo}
\end{figure*}
\subsection{GANs Need Not Converge to Nash Equilibrium}
In this section, we provide further empirical evidence for the claim that \textit{GANs can produce realistic data even at non-Nash critical points}. Section 3.2 of the main paper presented results for an SNGAN trained over the CIFAR-10 dataset. Figure \ref{fig:non_nash_demo} demonstrates similar results for SNGAN over the MNIST and CELEB-A datasets. We observe that the converged GAN configurations do not exhibit the characteristics of a Nash equilibrium, despite producing high fidelity samples. A Nash equilibrium is optimal for both the agents, thus no agent can deviate from it to unilaterally improve its payoff. As the generator (discriminator) aims to minimize (maximize) the objective function, a Nash equilibrium would constitute a local minima (maxima) for the generator (discriminator). The hessian of the objective function w.r.t the generator (discriminator) would thus be positive (negative) definite. However, as depicted in Figure \ref{fig:non_nash_demo}, the hessian of the objective function w.r.t the generator is indefinite as it has both positive as well as negative eigenvalues, indicating that the configuration is not a local minima for the generator and thus not a Nash equilibrium. This is also verified by the visualization (Figure \ref{fig:non_nash_demo} that the generator is able to deviate from the converged configuration on unilaterally optimizing the objective function, attaining a lower loss, but deteriorating the quality of the learned data distribution.  

\section{Theorems and Proofs}

% Let $G$ and $D$ represent generator and discriminator parameterized by $\theta_g$ and $\theta_d$ respectively. $V$ is the original GAN objective and $V^\lambda$ is the proximal GAN objective. The subscript $c$, $f$, and $w$ to original objective $V$ represents classical GAN, f-GAN, or Wasserstein GAN objectives respectively. $V_{G_w}$ and $V_{D_w}$ represent the worst generator and discriminator, respectively. The divergences are represented as $JSD$ for Jenson Shanon divergence, $W_c$ for Wasserstein distance, and $D_f$ for the $f$-divergence. $DG$ refers to vanailla duality gap and $DG^\lambda$ corresponds to proximal duality gap.

\subsection{Classical Duality Gap}
\begin{proposition}
\label{dg_zero_at_nash_only}
The duality gap (DG) for a GAN configuration will tend to zero only at a Nash equilibrium and is positive otherwise.
\end{proposition}
\label{dg_zero_at_nash_only}
\begin{proof}
 A configuration $(\theta_d^*,\theta_g^*)$ of the GAN game (Eq \ref{eq:gan_game}) is called a Nash equilibrium if and only if $\forall  \ \theta_d, \theta_g$,
 \begin{align*}
     V(D_{\theta_d},G_{\theta_g^*}) \leq V(D_{\theta_d^*},G_{\theta_g^*})  \leq  V(D_{\theta_d^*},G_{\theta_g})
 \end{align*}
Equivalently, $\underset{\Tilde{\theta_d} \in \Theta_D}{\max} \ V(D_{ \Tilde{\theta_d}},G_{\theta_g^*}) = \underset{\Tilde{\theta_g} \in \Theta_G}{\min} \ V(D_{\theta_d^*},G_{\Tilde{\theta_g}})$ 

We have from the definition of duality gap ($DG$),
\begin{align*}
DG(\theta_d,\theta_g)=  \underset{\Tilde{\theta_d} \in \Theta_D}{\max} \ V(D_{ \Tilde{\theta_d}},G_{\theta_g}) - \underset{\Tilde{\theta_g} \in \Theta_G}{\min} \ V(D_{\theta_d},G_{\Tilde{\theta_g}}) 
\end{align*}
Thus, when $DG(\theta_d,\theta_g) = 0 $
\begin{align*}
 & \implies  \underset{\Tilde{\theta_d} \in \Theta_D}{\max}V(D_{ \Tilde{\theta_d}},G_{\theta_g}) = \underset{\Tilde{\theta_g} \in \Theta_G}{\min}V(D_{\theta_d},G_{\Tilde{\theta_g}})\\
& \implies (\theta_d,\theta_g) \  \text{ is a Nash equilibrium.}
\end{align*}
Similarly, when $ (\theta_d,\theta_g) $  is a Nash equilibrium,
\begin{align*}
 & \implies  \underset{\Tilde{\theta_d} \in \Theta_D}{\max}V(D_{ \Tilde{\theta_d}},G_{\theta_g}) = \underset{\Tilde{\theta_g} \in \Theta_G}{\min}V(D_{\theta_d},G_{\Tilde{\theta_g}})\\
& \implies DG(\theta_d,\theta_g) = 0 
\end{align*}
When $ (\theta_d,\theta_g) $ does not constitute a Nash equilibrium,
\begin{align*}
 & \implies  \underset{\Tilde{\theta_d} \in \Theta_D}{\max}V(D_{ \Tilde{\theta_d}},G_{\theta_g}) > \underset{\Tilde{\theta_g} \in \Theta_G}{\min}V(D_{\theta_d},G_{\Tilde{\theta_g}})\\
& \implies DG(\theta_d,\theta_g) > 0 
\end{align*}
\end{proof}
However, as discussed, GANs can converge to non-Nash critical points while producing data samples of high fidelity. The behaviour of $DG$ at such scenarios is not well understood, limiting its applicability as a tool for monitoring GAN training.
\subsection{Proximal Duality Gap}
% \begin{lemma}
% \begin{align*}
% \max_{\theta_d^{'} \in \Theta_D}V^{\lambda}(D_{\theta_d^{'}},G_{\theta_g}) = \max_{\theta_d^{'} \in \Theta_D}\ V(D_{\theta_d^{'}},G_{\theta_g})
% \end{align*}
% \end{lemma}

% \begin{proof}
% \begin{align*}
% \textcolor{blue}{\mbox{We know that, } 
% V_c =  \mathop{\mathbb{E}_{ \textbf{x} \sim P_{r}}}[\log D(\textbf{x})]   + \mathop{\mathbb{E}_{\textbf{x} \sim P_{{\theta_g}}}}[\log (1-D(\textbf{x}))]}\\
% \end{align*}
% \end{proof}

% \begin{remark}
% For all proximal equilibria ($\theta_d^*,\theta_g^*$) of the GAN game \ref{eq:gan_game}, $V_{D_w}(\theta_g^*)=V^{\lambda}_{G_w}(\theta_d^*)=V^{\lambda}(\theta_d^*,\theta_g^*)$, thus  $DG^{\lambda}(\theta_d^*,\theta_g^*)=0$.
% \end{remark}

\begin{proposition}
\label{lemma:proximal_lessthan_original}
Consider a GAN game governed by the objective function $V$. Then, for a configuration ($\theta_d,\theta_g$) the proximal objective $V^{\lambda}$ is related to $V$ as : 
\begin{align*}
    V^{\lambda}(\theta_d,\theta_g) &\leq V_{D_w}(\theta_g)
\end{align*}
\end{proposition}

\begin{proof}
Since $\lambda ||D_{\Tilde{\theta_d}} - D_{\theta_d}||^2 \geq 0$, we have
\begin{align*}
    V(D_{\Tilde{\theta_d}},G_{\theta_g})-\lambda ||D_{\theta_d} - D_{\Tilde{\theta_d}}||^2 &\leq V(D_{\Tilde{\theta_d}},G_{\theta_g}) \\
    \implies V^{\lambda}(\theta_d,\theta_g) &\leq \max_{\Tilde{\theta_d}} V(D_{\Tilde{\theta_d}},G_{\theta_g})\\
    &= V_{D_w}(\theta_g)
\end{align*}
\end{proof}

\begin{lemma}
\label{lemma:m1}
Given a generator $\theta_g$, $V_{D_w}$ is related to the divergences between $P_r$ and $P_{\theta_g}$ in the various GAN objectives as follows
\begin{equation*}
    V_{D_w}(\theta_g)=
    \begin{cases}
      JSD(P_{\theta_g}||P_r) - \log 2, & \text{if}\ V=V_c \\
      W_c(P_{\theta_g}||P_r), & \text{if}\ V=V_w \\
      D_f(P_{\theta_g}||P_r), & \text{if}\ V=V_f \\
    \end{cases}
\end{equation*}
\end{lemma}

\begin{proof}
We divide the proof into three parts corresponding to each of the three GAN formulations.
\setcounter{casep}{0}
\begin{casep}
\textit{Classic GAN}\\
% \textit{For a classic GAN, i.e. $V=V_c$ , the value of the objective function corresponding to the worst case discriminator ($V_{D_w}$) is related to the Jenson-Shannon divergence ($JSD$) between the real ($P_r$) and generated data distributions ($P_{\theta_g}$) as $V_{D_w}(\theta_g) = 2JSD(P_{\theta_g}||P_r) - \log4$ } \\
Consider the classic GAN objective $V=V_c$. We have,
\begin{align*}
V &= \dfrac{1}{2}\mathop{\mathbb{E}_{ x \sim P_{r}}}[\log D(x)]   + \dfrac{1}{2}\mathop{\mathbb{E}_{x \sim P_{{\theta_g}}}}[\log (1-D(x))]\\
&= \dfrac{1}{2}\int P_r(x) \log D(x) dx + \dfrac{1}{2}\int P_{\theta_g}(x) \log (1-D(x)) dx 
\end{align*}
We have $V_{D_w} = \underset{D}{\max} \ V $. The worst case discriminator $D_w$ can be obtained by differentiating $V$ w.r.t $D$ for every $x$ and equating to zero. This gives:
\begin{align*}
D_w(x) = \dfrac {P_r(x)}{P_{\theta_g}(x)+P_r(x)}
\end{align*}
Substituting $D_w$ back into $V$ gives,
\begin{align*}
V_{D_w}  =& \dfrac{1}{2}\int P_r(x) \log \left( \dfrac{P_r}{P_r+P_{\theta_g}}\right) dx \\
                 & \ \ \ \ \ \ \ \ + \dfrac{1}{2}\int P_{\theta_g}(x) \log \left( \dfrac{P_{\theta_g}}{P_r+P_{\theta_g}}\right) dx \\ =& \ JSD(P_{\theta_g}||P_r) - \log2
\end{align*}
\end{casep}

\begin{casep}
\textit{Wasserstein GAN}\\
Consider the Wasserstein GAN objective $V=V_w$. We have, 
\begin{align*}
V &= \E_{x \sim {P}_{r}}[D(x)] - \mathbb{E}_{x \sim {P}_{\theta_g}}[D^c(x)]
\end{align*}
where, $D^{c}(x)$ is related to $D(x)$ as 
\begin{align*}
D^c(x) &= \underset{x'}{\sup} \  \{ \  D(x') - c(x,x') \ \} \\
&\geq D(x) - c(x,x) \\
&\geq D(x)
\end{align*}
Substituting for $D^c(x)$ into $V$, we have
\begin{align*}
V &\leq \E_{x \sim {P}_{r}}[D(x)] - \mathbb{E}_{x \sim {P}_{\theta_g}}[D(x)]
\end{align*}
Thus, if $\exists$ a $c-$concave function $D_w$ such that $D_w(x)=D^{c}_w(x)$, then the bound is attainable and we have,
\begin{align*}
V_{D_w} &= \underset{D \ c-concave}{\max} \  V  =  \underset{D \ c-concave}{\sup} \  V \\
&= \E_{x \sim {P}_{r}}[D_w(x)] - \mathbb{E}_{x \sim {P}_{\theta_g}}[D_w(x)]
\end{align*}
Consider a constant discriminator $D_{constant}(x)$ = $k$, which by definition satisfies $c-$concavity. We have,
\begin{align*}
D^c_{constant}(x) &= \underset{x'}{\sup} \  \{ \  D_{constant}(x') - c(x,x') \ \} \\
&= \underset{x'}{\sup} \  \{ \  k - c(x,x') \ \} \\
&=   k + \underset{x'}{\sup} \  \{ \ - c(x,x') \ \} \\
&= k = D_{constant}(x)
\end{align*}
Thus, for $D_w=D_{constant}$ the bound is attainable and we have,
\begin{align*}
V_{D_w} &= \underset{D \ c-concave}{\max} \  V  \\
&=  \underset{D \ c-concave}{\sup} \  \E_{x \sim {P}_{r}}[D(x)] - \mathbb{E}_{x \sim {P}_{\theta_g}}[D^c(x)] \\
&= W_c(P_{\theta_g}||P_r)
\end{align*}
\end{casep}

\begin{casep}
\textit{F-GAN}\\
Consider the F-GAN objective $V=V_f$. We have,
\begin{align*}
    V=\mathop{\mathbb{E}_{x \sim P_{r}}}[D(x)]   - \mathop{\mathbb{E}_{x \sim P_{{\theta_g}}}}[f^*(D(x))],
\end{align*}
where $f^*$ is the Fenchel conjugate of a convex function $f$ defined by $f^*(x)= \underset{t}{\max} \  \{ \ xt -f(t) \  \}$ . The maximum implied by $f^*$ can be obtained by differentiating it w.r.t $t$ and equating to zero. This gives
\begin{align*}
 x - f'(t) =& \ 0\\
\implies x =& \ f'(t)
\end{align*}
On Substituting the value of $x$ in $f^*(x)$, we get that $f^*$ satisfies the property
\begin{align}
    \label{fenchel_conjugate_of_derivative_property}
    f^*(f'(t))= tf'(t)- f(t)
\end{align}

We have $V_{D_w} = \underset{D}{\max} \ V $. The worst case discriminator $D_w$ can be obtained by differentiating $V$ w.r.t $D$ for every $x$ and equating to zero. This gives:
\begin{align*}
D_w(x) &= f^{*'-1}\left(\dfrac{P_r(x)}{P_{\theta_g}(x)}\right)
\end{align*}
Substituting $D_w$ back into $V$ we get,
\begin{align*}
V_{D_w} = \int & P_r(x) f^{*'-1}\left(\dfrac{P_r(x)}{P_{\theta_g}(x)}\right)dx  \\ & - \int P_{\theta_g}(x) f^*\left(f^{*'-1}\left(\dfrac{P_r(x)}{P_{\theta_g}(x)}\right) \right)dx
\end{align*}
The Fenchel conjugate $f^*$ of a convex function $f$ satisfies $f^{*'-1} = f'$. Thus, substituting for $f^{*'-1}$ in $V_{D_w}$, we get,
\begin{align*}
V_{D_w} &= \int P_r(x) f^{'}\left(\dfrac{P_r(x)}{P_{\theta_g}(x)}\right)dx  \\ & \hspace{30pt} - \int P_{\theta_g}(x) f^*\left(f^{'}\left(\dfrac{P_r(x)}{P_{\theta_g}(x)}\right) \right)dx \\
&= \int P_r(x) f^{'}\left(\dfrac{P_r(x)}{P_{\theta_g}(x)}\right)dx  \\ & \hspace{30pt} - \int P_{\theta_g}(x) \left( \dfrac{P_r(x)}{P_{\theta_g}(x)}f^{'}\left(\dfrac{P_r(x)}{P_{\theta_g}(x)}\right) \right) dx \\ & \hspace{30pt} + \int P_{\theta_g}(x) \left( f\left(\dfrac{P_r(x)}{P_{\theta_g}(x)}\right) \right) dx  \ \ \  \text{(using \ref{fenchel_conjugate_of_derivative_property})} \\
&= \int P_{\theta_g}(x)f\left(\dfrac{P_r(x)}{P_{\theta_g}(x)}\right) dx \\
&=  D_f(P_{\theta_g}||P_r)
\end{align*}
\end{casep}
\end{proof}

\begin{theorem}
\label{pdg_lower_bound}

Consider a GAN game governed by an objective function $V$. Then the proximal duality gap ($DG^\lambda$) at a configuration ($\theta_d,\theta_g$) is related to the  divergence between the real ($P_r$) and generated ($P_{\theta_g}$) data distributions as follows.
\begin{equation*}
    DG^{\lambda}(\theta_d,\theta_g) \geq DIV(P_{\theta_g}||P_r) - \kappa
\end{equation*}
where,
\begin{equation*}
    DIV(P_{\theta_g}||P_r)=
    \begin{cases}
      JSD(P_{\theta_g}||P_r), & \text{if}\ V=V_c \\
      W_c(P_{\theta_g}||P_r), & \text{if}\ V=V_w \\
      D_f(P_{\theta_g}||P_r), & \text{if}\ V=V_f \\
    \end{cases}
\end{equation*}
 and $\kappa \ (\geq0)$ denotes the minimum divergence that the considered class of generator functions can achieve with the real data distribution.

\end{theorem}
\begin{proof}
We divide the proof into three parts corresponding to each of the the three GAN formulations. For each GAN formulation, we denote by $\kappa$ the minimum divergence that the considered class of generator functions can attain with the real data distributions; $\min_{P_{\theta_g}} \ DIV(P_{\theta_g}||P_r)= \kappa$. Note that under the realizable setting, $\kappa = 0$ as $\exists \theta_g$ such that $P_{\theta_g} = P_r$. 

\setcounter{casep}{0}
\begin{casep}
\textit{Classic GAN}, \\
Consider the classic GAN objective $V=V_c$. We have, $V_{D_w}(\theta_g) = JSD(P_{\theta_g}||P_r) - \log2$ \hfill (lemma \ref{lemma:m1})

% Since $\lambda ||D_{\Tilde{\theta_d}} - D_{\theta_d}||^2 \geq 0$, we have
% \begin{align*}
%     V_{c}(D_{\Tilde{\theta_d}},G_{\theta_g})-\lambda ||D_{\theta_d} - D_{\Tilde{\theta_d}}||^2 &\leq V_{c}(D_{\Tilde{\theta_d}},G_{\theta_g}) \\
%     \implies V^{\lambda}_c(\theta_d,\theta_g) &\leq \max_{\Tilde{\theta_d}} V_{c}(D_{\Tilde{\theta_d}},G_{\theta_g})
% \end{align*}
Further, $V^{\lambda}(\theta_d,\theta_g) \leq V_{D_w}(\theta_g)$ \hfill (proposition \ref{lemma:proximal_lessthan_original})\\
 $\implies V^{\lambda}(\theta_d,\theta_g) \leq JSD(P_{\theta_g}||P_r) - \log2$\\
Thus, using a slight misuse of notation, we have
\begin{align*}
    V_{G_w}^{\lambda}(\theta_d) &= \min_{\theta_g^{'}} \ V^{\lambda}(\theta_d,\theta_g^{'})\\
    &\leq \min_{P_{\theta_g^{'}}} \  JSD(P_{\theta_g^{'}}||P_r) - \log2\\
    &= \kappa - \log2 \\
    \therefore DG^{\lambda}(\theta_d,\theta_g) &=  V_{D_w}(\theta_g) - V_{G_w}^{\lambda}(\theta_d)\\
    %&\geq 2JSD(P_{\theta_g}||P_r) - \log4 + \log4\\
    &\geq JSD(P_{\theta_g}||P_r) - \kappa
\end{align*}
\end{casep}

\begin{casep}
\textit{Wasserstein GAN}, \\
Consider the Wasserstein GAN objective $V=V_w$. We have, $V_{D_w}(\theta_g) = W_c(P_{\theta_g}||P_r)$ \hfill (lemma \ref{lemma:m1})

Further, $V^{\lambda}(\theta_d,\theta_g) \leq V_{D_w}(\theta_g)$ \hfill (proposition \ref{lemma:proximal_lessthan_original})\\
 $\implies V^{\lambda}(\theta_d,\theta_g) \leq W_c(P_{\theta_g}||P_r)$\\
Thus, using a slight misuse of notation, we have
\begin{align*}
    V_{G_w}^{\lambda}(\theta_d) &= \min_{\theta_g^{'}} \ V^{\lambda}(\theta_d,\theta_g^{'})\\
    &\leq \min_{P_{\theta_g^{'}}} \  W_c(P_{\theta_g^{'}}||P_r)\\
    &= \kappa \\
    \therefore DG^{\lambda}(\theta_d,\theta_g) &=  V_{D_w}(\theta_g) - V_{G_w}^{\lambda}(\theta_d)\\
    %&\geq W_c(P_{\theta_g}||P_r) - 0\\
    &\geq W_c(P_{\theta_g}||P_r) - \kappa
\end{align*}
\end{casep}

\begin{casep}
\textit{F-GAN},\\
Consider the F-GAN objective $V=V_f$. We have, $V_{D_w}(\theta_g) = D_f(P_{\theta_g}||P_r)$ \hfill (lemma \ref{lemma:m1})

Further, $V^{\lambda}(\theta_d,\theta_g) \leq V_{D_w}(\theta_g)$ \hfill (proposition \ref{lemma:proximal_lessthan_original})\\
 $\implies V^{\lambda}(\theta_d,\theta_g) \leq D_f(P_{\theta_g}||P_r)$\\
Thus, using a slight misuse of notation, we have
\begin{align*}
    V_{G_w}^{\lambda}(\theta_d) &= \min_{\theta_g^{'}} \ V^{\lambda}(\theta_d,\theta_g^{'})\\
    &\leq \min_{P_{\theta_g^{'}}} \  D_f(P_{\theta_g^{'}}||P_r)\\
    &= \kappa \\
    \therefore DG^{\lambda}(\theta_d,\theta_g) &=  V_{D_w}(\theta_g) - V_{G_w}^{\lambda}(\theta_d)\\
   % &\geq D_f(P_{\theta_g}||P_r) - 0\\
    &\geq D_f(P_{\theta_g}||P_r) - \kappa
\end{align*}
\end{casep}
\end{proof}

\begin{theorem}
\label{thm:pdg_zero_stackelberg}
The proximal duality gap ($DG^{\lambda}$) at a configuration  ($\theta_d^*,\theta_g^*$) for the  GAN game defined by $V_c, V_w,$ or $V_f$ is equal to zero for $\lambda=0$, when the generator learns the real data distribution .i.e,  $P_{\theta_g^*}=P_r \implies DG^{\lambda=0}(\theta_d^*,\theta_g^*) = 0$.
\end{theorem}
\begin{proof}
Let $(\theta_d^*,\theta_g^*)$ be a GAN configuration such that $P_{\theta_g^*}=P_r$. Since this is a realizable setting, $\min_{P_{\theta_g}} \ DIV(P_{\theta_g}||P_r)= 0$. We divide the proof into three parts corresponding to each of the the three GAN formulations.
\setcounter{casep}{0}
\begin{casep}
\textit{Classic GAN}, \\
Consider the classic GAN objective $V=V_c$. We have,
\begin{align*}
 V_{D_w}(\theta_g^*) &= JSD(P_{\theta_g^*}||P_r) - \log2 \ \ \ \ \  \text{(lemma \ref{lemma:m1})} \\
   &=  - \log2  \ \ \ \  (\because P_{\theta_g^*} = P_r)\\ 
    V_{G_w}^{\lambda=0}(\theta_d^*) &= \min_{\theta_g^{'}} \ V^{\lambda=0}(\theta_d^*,\theta_g^{'})\\
    &= \min_{P_{\theta_g^{'}}} \ \max_{\Tilde{\theta_d}} \  V(\Tilde{\theta_d},\theta_g^{'})\\
    &= \min_{P_{\theta_g^{'}}} \ V_{D_w}(\theta_g^{'}) \\
    &= \min_{P_{\theta_g^{'}}} \ JSD(P_{\theta_g^{'}}||P_r) - \log2 \\
    &= - \log2 \\
\therefore DG^{\lambda=0}(\theta_d^*,\theta_g^*) &= V_{D_w}(\theta_g^*) - V_{G_w}^{\lambda=0}(\theta_d^*) \\
&= \ 0
\end{align*}
\end{casep}

\begin{casep}
\textit{Wasserstein GAN}, \\
 Consider the Wasserstein GAN objective $V=V_w$. We have,
\begin{align*}
 V_{D_w}(\theta_g^*) &=  W_c(P_{\theta_g^*}||P_r) \ \ \ \ \ \ \   \text{(lemma \ref{lemma:m1})} \\
   &=  0  \ \ \ \  (\because P_{\theta_g^*} = P_r)\\ 
    V_{G_w}^{\lambda=0}(\theta_d^*) &= \min_{\theta_g^{'}} \ V^{\lambda=0}(\theta_d^*,\theta_g^{'})\\
    &= \min_{P_{\theta_g^{'}}} \ \max_{\Tilde{\theta_d}} \  V(\Tilde{\theta_d},\theta_g^{'})\\
    &= \min_{P_{\theta_g^{'}}} \ V_{D_w}(\theta_g^{'}) \\
    &= \min_{P_{\theta_g^{'}}} \ W_c(P_{\theta_g^{'}}||P_r) \\
    &= 0 \\
\therefore DG^{\lambda=0}(\theta_d^*,\theta_g^*) &= V_{D_w}(\theta_g^*) - V_{G_w}^{\lambda=0}(\theta_d^*) \\
&= \ 0
\end{align*}
\end{casep}
\begin{casep}
\textit{F-GAN}, \\
 Consider the F-GAN objective $V=V_f$. We have,
\begin{align*}
 V_{D_w}(\theta_g^*) &=  D_f(P_{\theta_g^*}||P_r) \ \ \ \ \ \ \   \text{(lemma \ref{lemma:m1})} \\
   &=  0  \ \ \ \  (\because P_{\theta_g^*} = P_r)\\ 
    V_{G_w}^{\lambda=0}(\theta_d^*) &= \min_{\theta_g^{'}} \ V^{\lambda=0}(\theta_d^*,\theta_g^{'})\\
    &= \min_{P_{\theta_g^{'}}} \ \max_{\Tilde{\theta_d}} \  V(\Tilde{\theta_d},\theta_g^{'})\\
    &= \min_{P_{\theta_g^{'}}} \ V_{D_w}(\theta_g^{'}) \\
    &= \min_{P_{\theta_g^{'}}} \ D_f(P_{\theta_g^{'}}||P_r) \\
    &= 0 \\
\therefore DG^{\lambda=0}(\theta_d^*,\theta_g^*) &= V_{D_w}(\theta_g^*) - V_{G_w}^{\lambda=0}(\theta_d^*) \\
&= \ 0
\end{align*}
\end{casep}
\end{proof}
\begin{corollary}
For the GAN formulations defined by $V_c, V_w$ or $V_f$, the generator learns the real data distribution at a configuration ($\theta_d^*,\theta_g^*$) if and only if ($\theta_d^*,\theta_g^*$) constitutes a Stackelberg equilibrium.
\end{corollary}
\begin{proof}
Consider a configuration ($\theta_d^*,\theta_g^*$) for the GAN game defined by $V_c, V_w$ or $V_f$. We have,
\setcounter{casep}{0}
\begin{casep}
\textit{When $P_{\theta_g^*} = P_r$}, 
\begin{align*}
\text{Theorem \ref{thm:pdg_zero_stackelberg}} &\implies DG^{\lambda=0}(\theta_d^*,\theta_g^*) = 0 \\
&\implies (\theta_d^*,\theta_g^*) \in \text{Stackelberg Equilibria}
\end{align*}
\end{casep}
\begin{casep}
\textit{When $(\theta_d^*,\theta_g^*) \in $ Stackelberg Equilibria}, \\
From the definition of $DG^{\lambda}$ and Stackelberg Equilibrium,
\begin{align*}
&DG^{\lambda=0}(\theta_d^*,\theta_g^*) = 0 \\
&\implies DIV(P_{\theta_g^*}||P_r) \leq 0  \ \ \  (\text{Theorem \ref{pdg_lower_bound}})\\
&\implies DIV(P_{\theta_g^*}||P_r) = 0 \ \ \ (\because DIV \geq 0) \\
&\implies P_{\theta_g^*} = P_r
\end{align*}
\end{casep}
\end{proof}
% \begin{proof}
% $V_{cD_w}= 2JSD (P_{\theta_g}||P_r) -\log 4 $ \hfill (lemma\ref{lemma:classic_gan_m1})\\
% And, $V_{cG_w}^{\lambda}(\theta_d) &=\min_{\Tilde{\theta_g}}V^{\lambda}(D_{\theta_d},G_{\Tilde{\theta_g}})$\\
% \begin{align*}
%  V_{cG_w}^{\lambda}(\theta_d) &= \underset{\Tilde{\theta_g}}{\min}\{\underset{\Tilde{\theta_d}}{\max} \ V(\Tilde{\theta_d},\theta_g)-\lambda||D_{\Tilde{\theta_d}}-D_{\theta_d}||^2\} \\
% If \lambda \rightarrow 0\\ 
%  V_{cG_w}^{\lambda}(\theta_d) &= \underset{\Tilde{\theta_g}}{\min}\{\underset{\Tilde{\theta_d}}{\max} \ V(\Tilde{\theta_d},\theta_g)\} \\
% \implies  \underset{\Tilde{\theta_g}}{\min}\{V_{cD_w}\} \\
% \implies  \underset{\Tilde{\theta_g}}{\min}\{2JSD (P_{\theta_g}||P_r) -\log 4 \} \\
% \implies  -\log 4\\
%  DG^{\lambda}_{c}(\theta_d,\theta_g) &=  V_{cD_w}(\theta_g) - V_{cG_w}^{\lambda}(\theta_d)\\
% %  \implies 2JSD (P_{\theta_g}||P_r) -\log 4 + \log 4\\
% %  \implies 2JSD (P_{\theta_g}||P_r)\\
%  \implies 0 \ \{\because  P_r = P_{{\theta_g}}\}
% \end{align*}
% The proof follows for other GANs as well.
% \end{proof}
\begin{theorem}
\label{thm:hierarchical_pdg}
Consider a GAN configuration $(\theta_d,\theta_g)$. Then,  $\forall \lambda^{'} \geq \lambda_{0}$,
\begin{equation*}
DG^{\lambda=\lambda^{'}}(\theta_d,\theta_g) = 0 \implies DG^{\lambda=\lambda_{0}}(\theta_d,\theta_g) = 0 
\end{equation*}
\end{theorem}
\begin{proof}
We know from the definition of $DG^{\lambda}$ that $DG^{\lambda}(\theta_d,\theta_g)=0$ is a necessary and sufficient condition for $(\theta_d,\theta_g)$ to be a $\lambda-$proximal equilibrium i.e. , $DG^{\lambda}(\theta_d,\theta_g) = 0$ implies that $\forall \  \theta_d^{'},\theta_g^{'}$,
\begin{align*}
    V(D_{\theta_d^{'}},G_{\theta_g}) & \leq V(D_{\theta_d},G_{\theta_g}) \\
      & \leq  \underset{\Tilde{\theta_d} \in \Theta_{D}}{\max} \ V(D_{\Tilde{\theta_d}},G_{\theta_g^{'}}) - \lambda ||D_{\Tilde{\theta_d}} - D_{\theta_d}||^2 
\end{align*}
and vice-versa.\\
Now, $\forall \lambda^{'} \geq \lambda_{0}$, the following holds 
\begin{align*}
    \underset{\Tilde{\theta_d} \in \Theta_{D}}{\max} \ & V(D_{\Tilde{\theta_d}},G_{\theta_g^{'}}) - \lambda^{'} ||D_{\Tilde{\theta_d}} - D_{\theta_d}||^2 \\
    &\leq \underset{\Tilde{\theta_d} \in \Theta_{D}}{\max} \ V(D_{\Tilde{\theta_d}},G_{\theta_g^{'}}) - \lambda_{0} ||D_{\Tilde{\theta_d}} - D_{\theta_d}||^2 
\end{align*}
Thus, $(\theta_d,\theta_g)$ is a $\lambda^{'}-$proximal equilibrium $\implies$ $(\theta_d,\theta_g)$ is also  a $\lambda_{0}-$proximal equilibrium. \\
$\therefore DG^{\lambda=\lambda^{'}}(\theta_d,\theta_g) = 0 \implies DG^{\lambda=\lambda_{0}}(\theta_d,\theta_g) = 0 $
\end{proof}

\begin{theorem}
\label{thm:pdg_upperbound}
Consider a GAN game governed by an objective function $V$. For $\lambda > 0$, let $V^{\lambda}$ denote the proximal objective defined by $V^{\lambda}(\theta_d,\theta_g)=max_{\Tilde{\theta_d}}V(\Tilde{\theta_d},\theta_g)-\lambda||D_{\Tilde{\theta_d}}-D_{\theta_d}||^2$ . Then, $\forall \ \epsilon > 0, \ \exists \ \delta > 0$ such that if $||D_{\theta_d}-D_{\Tilde{\theta_d}}||<\delta$, then $DG^{\lambda}(\theta_d,\theta_g) - DIV(P_{\theta_g}||P_r)<\epsilon$
where,
\begin{equation*}
    DIV(P_{\theta_g}||P_r)=
    \begin{cases}
      JSD(P_{\theta_g}||P_r), & \text{if}\ V=V_c \\
      W_c(P_{\theta_g}||P_r), & \text{if}\ V=V_w \\
      D_f(P_{\theta_g}||P_r), & \text{if}\ V=V_f \\
    \end{cases}
\end{equation*}
\end{theorem}
\begin{proof}
 We show that for all $\epsilon > 0$ and $\lambda > 0$, $\delta = \sqrt{\dfrac{\epsilon}{\lambda}}$ satisfies the claim. We provide the proof for F-GAN. The proof for the other GAN formulations follow on the same lines.  \\
Consider a configuration ($\theta_d,\theta_g$) for the GAN game defined by the F-GAN objective $V=V_f$. We have, $V_{D_w}(\theta_g) = D_f(P_{\theta_g}||P_r)$ \hfill (lemma \ref{lemma:m1})\\
Given that $||D_{\theta_d}-D_{\Tilde{\theta_d}}||<\delta$, we have
\begin{align*}
DG^{\lambda}(\theta_d,\theta_g) &- DIV(P_{\theta_g}||P_r) \\
    &=V_{D_w}(\theta_g) - V_{G_w}^{\lambda}(\theta_d) - D_f(P_{\theta_g}||P_r)\\ 
    & =D_f(P_{\theta_g}||P_r) - V_{G_w}^{\lambda}(\theta_d)\\  & \hspace{30pt} - D_f(P_{\theta_g}||P_r) \\ 
    &= - V_{G_w}^{\lambda}(\theta_d) \\ 
    &= -\min_{\theta_g^{'}} \ V^{\lambda}(\theta_d,\theta_g^{'})\\
    &= -\min_{\theta_g^{'}} \  \{ \max_{\Tilde{\theta_d}} V(\Tilde{\theta_d},\theta_g^{'})  \\ & \hspace{90pt} - \lambda||D_{\Tilde{\theta_d}}-D_{\theta_d}||^2\} \\
    &< -\min_{\theta_g^{'}} \  \{ \max_{\Tilde{\theta_d}} V(\Tilde{\theta_d},\theta_g^{'}) - \lambda \delta^2 \}\\
    &= -\min_{\theta_g^{'}} \  \{ V_{D_w}(\theta_g^{'}) \} + \lambda \delta^2\\
    &= -\min_{P_{\theta_g^{'}}} \  \{ D_f(P_{\theta_g^{'}}||P_r) \}  + \lambda \delta^2\\
    &= \lambda \delta^2\\
    &= \lambda \left(\sqrt{\dfrac{\epsilon}{\lambda}}\right)^2\\
    &= \epsilon
\end{align*}
\end{proof}
\section{$\mathbf{DG^{\lambda}}$ Estimation}
\begin{algorithm}
\SetAlgoLined
\SetKwBlock{Begin}{function}
\textbf{Input:} GAN configuration - ($\theta_d^t$, $\theta_g^t$), data points $x_i$\\
\Begin(\textit{prox\_opt}{(}$\theta_d,\theta_g${)}\ :){
$\Tilde{\theta_d} \longleftarrow \ \theta_d$\\
\For{j = 0 to T}{
$V^{\lambda} \longleftarrow V(\Tilde{\theta_d}, \theta_g ) - \frac{\lambda}{n} \sum_{i=1}^{n}||\nabla_{x}D_{\Tilde{\theta_d}}(x_i)-\nabla_{x}D_{\theta_{d}}(x_i)||_{2}^{2}$\\
$\Tilde{\theta_d} \longleftarrow \ \Tilde{\theta_d} + \eta \nabla_{\Tilde{\theta_d}} V^{\lambda}$}
\textbf{return} $\Tilde{\theta_d}, \ V^{\lambda}$}    
$\theta_d^w \longleftarrow \ \theta_d^t \ ; \ \theta_g^w \longleftarrow \ \theta_g^t $\\
\For{i = 0 to N\_ITER}{
$\theta_d^w \longleftarrow \ \theta_d^w + \eta \nabla_{\theta_d^w} V(\theta_d^w, \theta_g^t)$ \\
$\theta_d^*, V^{\lambda} \longleftarrow $\textit{prox\_opt}($\theta_d^t,\theta_g^w$)\\
$\theta_g^w \longleftarrow \ \theta_g^w - \eta \nabla_{\theta_g^w} V(\theta_d^*, \theta_g^w)$}
% $(\theta_g^{w}, \theta_d^{w})$ are the converged worst case configurations\\
$V_{D_w} \longleftarrow V(\theta_g^t , \theta_d^{w})$\\
$\theta_d^*, V^{\lambda}_{G_w} \longleftarrow $ \textit{prox\_opt}$(\theta_d^t,\theta_g^{w})$\\
\textbf{return} $DG^{\lambda}(\theta_d^t,\theta_g^t)$ = $V_{D_w} - V^{\lambda}_{G_w}$\\
\caption{Proximal Duality Gap\ $DG^{\lambda}(\theta_d^t, \theta_g^t$)}
\label{alg:proximal_dg}
\end{algorithm}

Algorithm \ref{alg:proximal_dg} summarizes the estimation process for proximal duality gap. Given a configuration $(\theta_d,\theta_g)$ of the GAN, we estimate $V_{D_w}$ and $V^{\lambda}_{G_w}$ by optimizing the objective function w.r.t the individual agents using gradient descent. We have the proximal objective defined by,
\begin{equation*}
\label{proximal_objective}
V^{\lambda}(D_{\theta_d},G_{\theta_g}) = \underset{\Tilde{\theta_d} \in \Theta_{D}}{\max} \ V(D_{\Tilde{\theta_d}},G_{\theta_g}) - \lambda ||D_{\Tilde{\theta_d}} - D_{\theta_d}||^2
\end{equation*}
Following (Farnia \& Ozdaglar, 2020) we use the Sobolev norm in $V^{\lambda}$, given by 
\begin{equation*}
\label{sobolev}
||D|| = \sqrt{\mathbb{E}_{\textbf{x} \sim {P}_r}\Big[||\nabla_{\textbf{x}}D(\textbf{x})||^{2}_{2}\Big]} 
\end{equation*}
The estimation process for $DG^{\lambda}$ is similar to that of $DG$, except that the worst case generator for a given discriminator is computed w.r.t the proximal objective ($V^{\lambda}$). As depicted in Algorithm \ref{alg:proximal_dg}, the function \textit{prox\_opt} uses gradient ascent to estimate the proximal objective $V^{\lambda}$. Since $\lambda$ restricts the neighbourhood within which the discriminator is optimal in $V^{\lambda}$, the search space for the optimal discriminator increases as $\lambda$ decreases. Correspondingly, estimating $V^{\lambda}$ demands a larger number of gradient steps and becomes computationally infeasible as $\lambda \rightarrow 0$.

\begin{figure}[h]
    \begin{center}
     \includegraphics[width = 0.85 \linewidth]{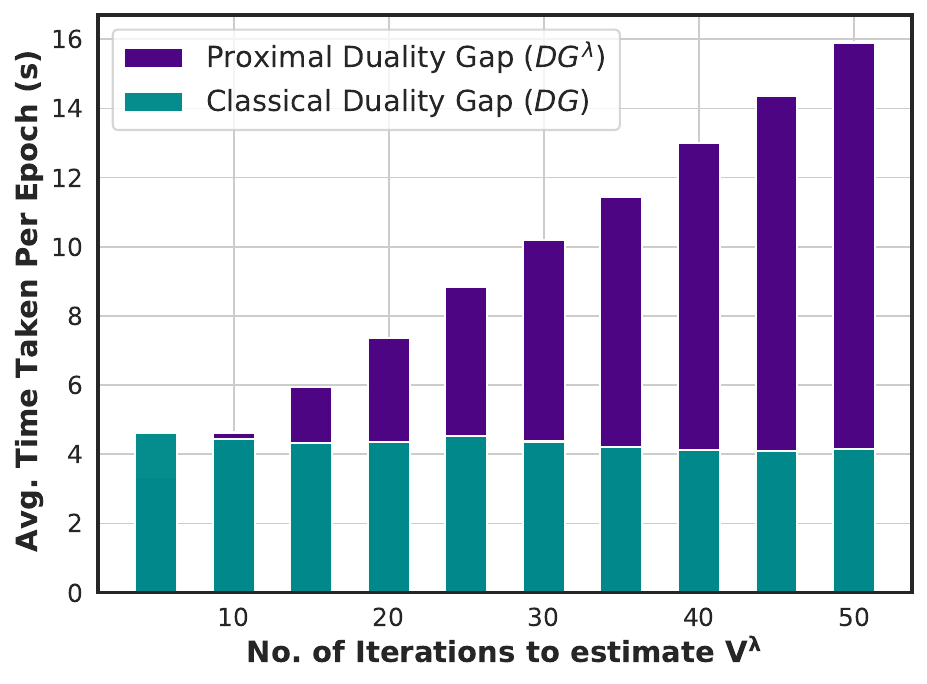}
    \end{center}
      \caption{Computational Complexity of $DG^{\lambda}$ over $DG$}
    \label{fig:complexity}
\end{figure}

We thus experimentally studied the computational overhead in estimating $DG^{\lambda}$ over $DG$. Figure \ref{fig:complexity} compares the average time taken per epoch to  estimate $DG$ and $DG^{\lambda}$ across varying gradient steps ($T$) to approximate $V^{\lambda}$. We observe that while $DG^{\lambda}$ has comparable computational complexity as $DG$ for smaller values of $T$, it increases rapidly for larger values of $T$. We observed that for $\lambda=0.1$, $\approx 20$ steps were sufficient for $V^{\lambda}$ to converge (in line with the observations of (Farnia \& Ozdaglar, 2020) ), incurring computational expense (Figure \ref{fig:complexity}) comparable to that demanded by $DG$.

\section{Experiments and Results}

 \begin{figure*}[h]
    \centering
    \begin{tabular}{cc}
    \subfloat[MNIST]{\includegraphics[width = 0.34\linewidth]{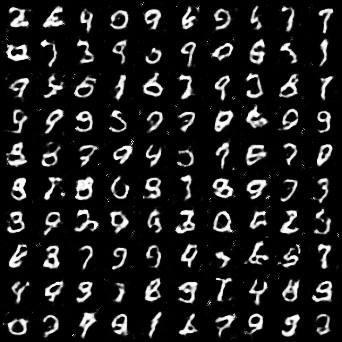} \ \includegraphics[width = 0.56\linewidth]{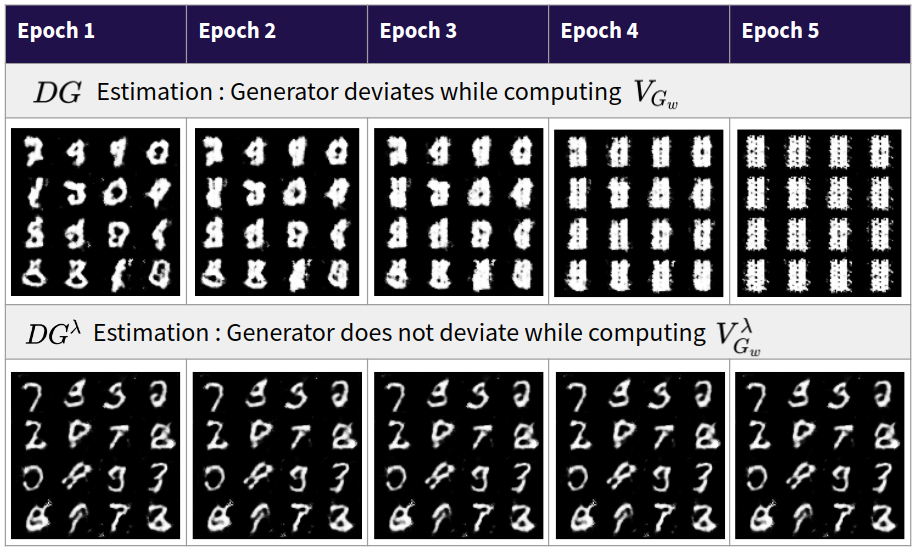}}\\ \\
    \subfloat[CIFAR-10]{\includegraphics[width = 0.34\linewidth]{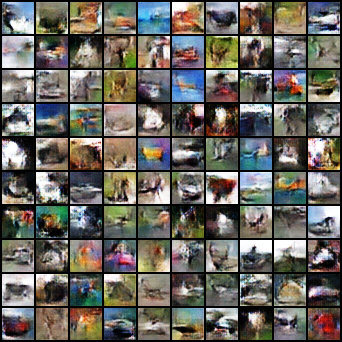} \ \includegraphics[width = 0.56\linewidth]{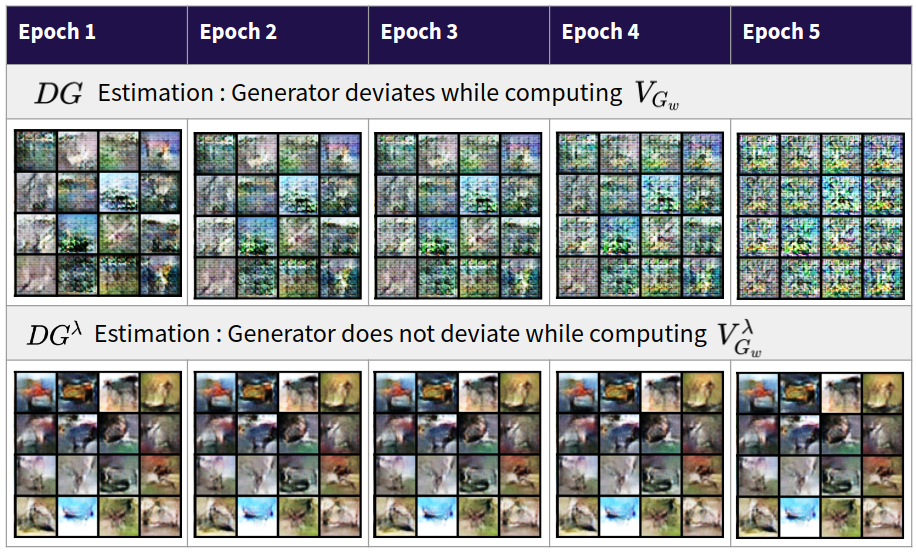}}\\ \\
    \subfloat[CELEB-A]{\includegraphics[width = 0.34\linewidth]{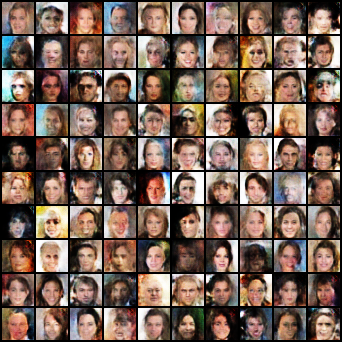} \ \includegraphics[width = 0.56\linewidth]{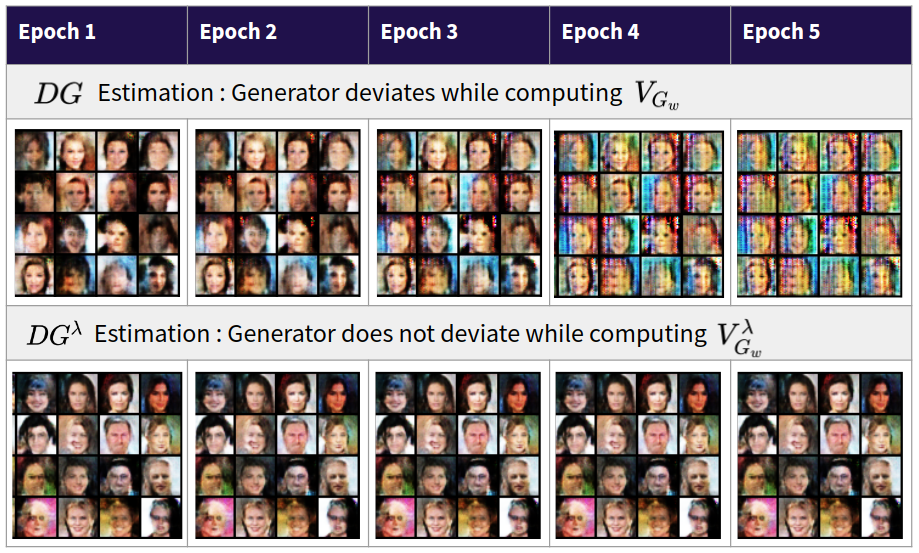}}\\ \\
    \end{tabular}
    \caption{Visualizing the behaviour the generator while estimating $DG$ and $DG^{\lambda}$ for a converged Wasserstein GAN (WGAN)  over the datasets- (a) MNIST (b) CIFAR-10 (c) CELEB-A.  The GAN has converged, validated by the high fidelity of generated data samples(left grid). However, the generator deviates on optimizing $V$, but remains stationary on optimizing $V^{\lambda}$, indicating that the converged configuration is a proximal equilibrium and not a Nash equilibrium.}
     \label{fig:wgan_output}
\end{figure*}

 \begin{figure*}[h]
    % \fbox{\rule{0pt}{2in} \rule{.9\linewidth}{0pt}}
    \centering
    
    % \begin{tikzpicture}
    %     \begin{axis}[%
    %     hide axis,
    %     legend columns=2,
    %     xmin=10,
    %     xmax=50,
    %     ymin=0,
    %     ymax=0.4,
    %     legend style={draw=white!15!black,legend cell align=right}
    %     ]
    %     \addlegendimage{crimson,line width=2.5 pt}
    %     \addlegendentry{Proximal Duality Gap };
    %     \addlegendimage{indigo,line width=2.5 pt,mark options={solid}}
    %     \addlegendentry{Classic Duality Gap };
    %     \end{axis}
    % \end{tikzpicture}
    
    \begin{tabular}{cc}
    \subfloat[MNIST]{\includegraphics[width = 0.34\linewidth]{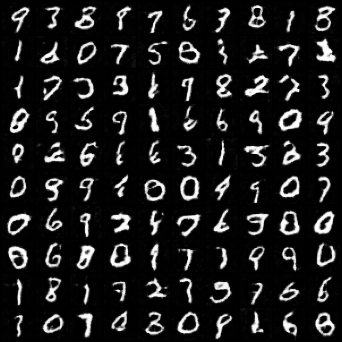} \ \includegraphics[width = 0.56\linewidth]{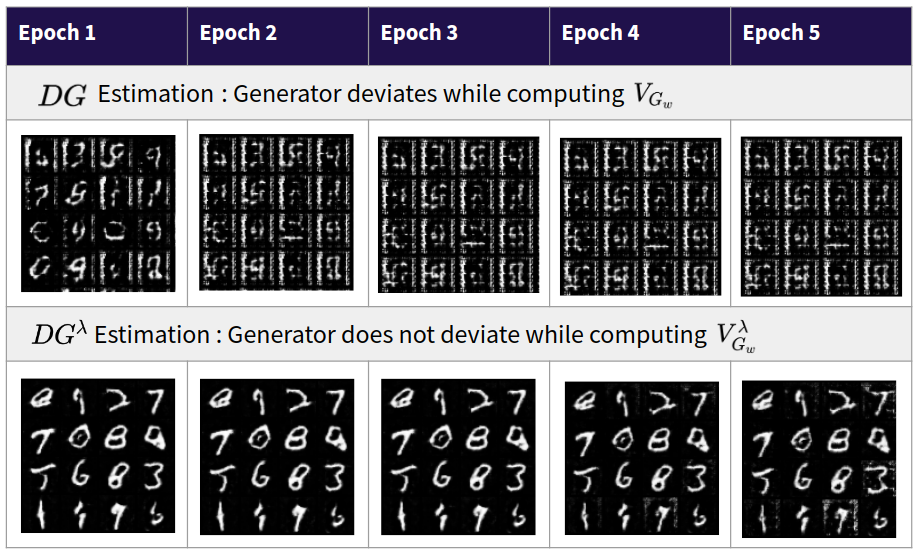}} \\
    \\ 
    \subfloat[CIFAR-10]{\includegraphics[width = 0.34\linewidth]{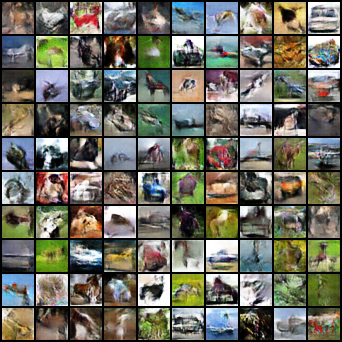} \ \includegraphics[width = 0.56\linewidth]{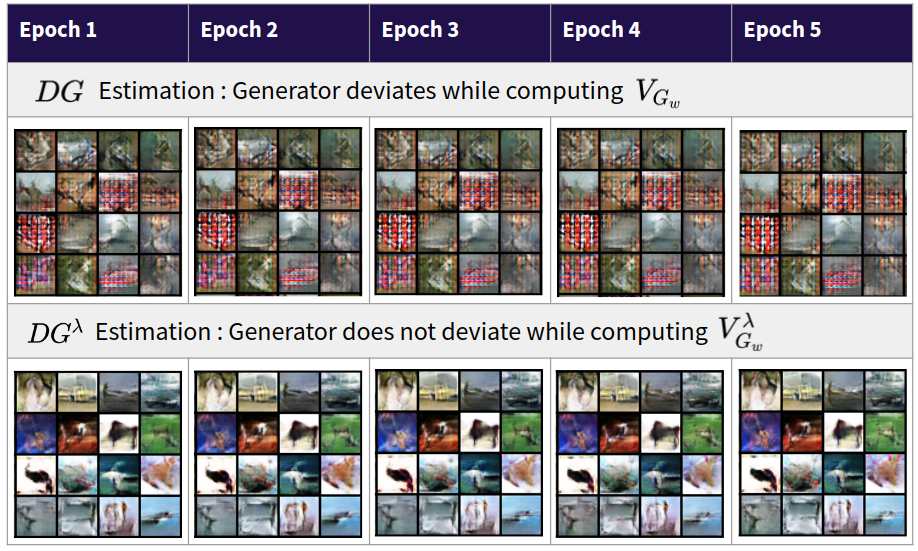}}\\ \\
    \subfloat[CELEB-A]{\includegraphics[width = 0.34\linewidth]{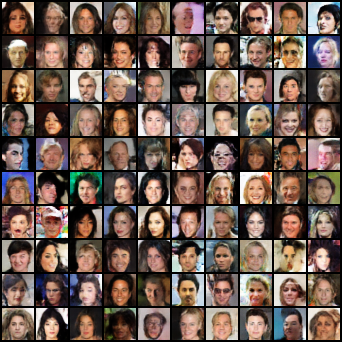} \ \includegraphics[width = 0.56\linewidth]{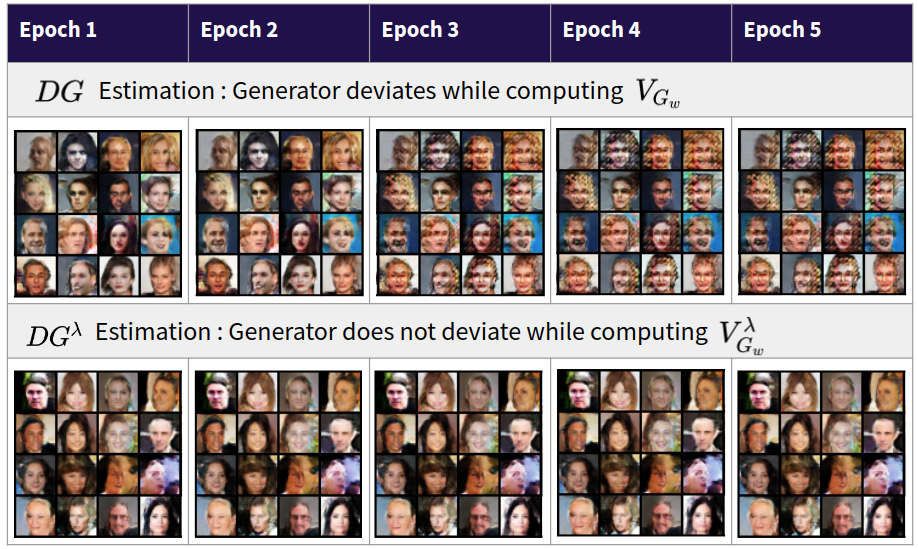}}\\ \\
    \end{tabular}
    \caption{Visualizing the behaviour the generator while estimating $DG$ and $DG^{\lambda}$ for a converged Spectral Norrmalized GAN (SNGAN) over the datasets- (a) MNIST (b) CIFAR-10 (c) CELEB-A.  The GAN has converged, validated by the high fidelity of generated data samples(left grid). However, the generator deviates on optimizing $V$, but remains stationary on optimizing $V^{\lambda}$, indicating that the converged configuration is a proximal equilibrium and not a Nash equilibrium.}
     \label{fig:sngan_output}
\end{figure*}

\begin{table}[h]
\centering
\begin{tabular}{@{}lcccc@{}}
\toprule
\textbf{}         & \multicolumn{4}{c}{\textbf{Pearson Correlation Coefficient (r)}}            \\ \midrule
                  & r$_{DG,IS}$ & r$_{DG^{\lambda},IS}$ & r$_{DG,FID}$ & r$_{DG^{\lambda},FID}$ \\ \midrule
\textit{MNIST}    & -0.104      & \textbf{-0.516}                & 0.207        &\textbf{0.942}                  \\
\textit{CIFAR-10} & -0.368      & \textbf{-0.463}                & 0.293        & \textbf{0.661}                \\
\textit{CELEB-A}  & -0.535      & \textbf{-0.846}                & 0.638        & \textbf{0.929}                 \\ \bottomrule
\end{tabular}%
\caption{Comparing the correlation of $DG$ and $DG^{\lambda}$ with IS and FID computed during the training of SNGAN over the 3 datasets.}
\label{sngan_correlation}
\end{table}

%  \begin{figure*}[]
%     \centering
%     \begin{tabular}{cc}
%     \subfloat[MNIST]{\includegraphics[width = 0.30\linewidth]{figures/gan_output/sagan/mnist.png}\includegraphics[width = 0.62\linewidth]{figures/monitoring/sngan/sngan_mnist_m2.png}}\\
%     \subfloat[CIFAR-10]{\includegraphics[width = 0.32\linewidth]{figures/gan_output/sagan/cifar.png}}\\
%     \subfloat[CELEB-A]{\includegraphics[width = 0.32\linewidth]{figures/gan_output/sagan/celeb.png}}\\
%     \end{tabular}
%     \caption{Self Attention GAN Convergence}
%      \label{fig:sagan_output}
% \end{figure*}
\subsection{Monitoring GAN Training Using $\mathbf{DG^{\lambda}}$}
In this section, we present further empirical observations and evidence to demonstrate the proficiency of proximal duality gap. Section 5.1 of the main paper presented the adeptness of $DG^{\lambda}$ over $DG$ to monitor GAN convergence, suggesting that the GANs have attained a proximal equilibrium and not a Nash equilibrium. We thus studied the behaviour of the converged GAN while estimating $DG$ and $DG^{\lambda}$. We visualized the variation in the generated data distribution  across epochs during the estimation of $V_{G_w}$ (for $DG$) and $V^{\lambda}_{G_w}$ (for $DG^{\lambda}$). The behaviours are demonstrated in Figure \ref{fig:wgan_output} for SNGAN and Figure \ref{fig:sngan_output} for WGAN across the three datasets - (a) MNIST, (b) CIFAR-10 and (c) CELEB-A. The high fidelity of generated data samples depicted in each subfigure suggest that the GANs have converged. However, on optimizing the objective function ($V$) unilaterally w.r.t the generator, it deviates (top row on the right) and deteriorates the quality of the learned data distribution, validating that the GAN has not attained a Nash equilibrium. However, the generator is unable to deviate (bottom row on the right) from the converged configuration w.r.t the proximal objective ($V^{\lambda}$) as the generated data distribution does not vary, indicating that the GANs have attained a proximal equilibrium. This explains why $DG^{\lambda}$ is able to better characterize convergence over $DG$ for each GAN. 

Table \ref{sngan_correlation} presents the correlation of $DG$ and $DG^{\lambda}$ with the popular quality evaluation measures IS and FID, across the training process of an SNGAN over each of the three datasets. We observe that $DG^{\lambda}$ has a higher correlation over $DG$ with each of the measures. Thus, further validating the claim that $DG^{\lambda}$ is adept to not only monitor convergence of the GAN game to an equilibrium, but also the goodness of the learned data distribution.

\subsection{Implementation and Hyperparameter Details}
We used the 4-layer DCGAN architecture for both the generator and the discriminator networks in all the experiments. We used an Adam optimizer to train and evaluate all the models. To compute $DG^{\lambda}$, we used $\lambda=0.1$ and $20$ optimization steps for approximating the proximal objective. To enforce the Lipschitz constraint in WGAN, we used weight clipping in the range $[-0.01,0.01]$. We used a batch size of $512,512,128$ for MNIST, CIFAR-10 and CELEB-A datasets respectively, where the input images were resized to be of shape $32\times32$. The latent space dimension for the generator was set to $100$.  To ensure that we obtain an unbiased estimate for $DG^{\lambda}$ and $DG$, we split each dataset into 3 disjoint sets - $S_A$, $S_B$ and $S_C$. We trained the GAN using $S_A$, we used $S_B$ to find the worst case counter parts $D_w$ and $G_w$ via gradient descent, and $S_C$ to evaluate the objective function at the obtained worst case configurations. For each dataset, we kept $5000$ samples each in $S_B$, $S_C$ and the rest in $S_A$. To estimate the worst case configurations, we optimized each agent unilarerally for 10 epochs over $S_B$. The learning rates for the discriminator ($LR_D$) and generator ($LR_G$), the value in multiples of which the DCGAN architecture steps up the convolutional features ($Step  \ Channels$) and the values of ($\beta_1,\beta_2$) used in the Adam optimizer for each of the datasets are summarized in Table \ref{hp_wgan} for WGAN and Table \ref{hp_sngan} for SNGAN.

\begin{table}[h]
\centering
\begin{tabular}{@{}lccccc@{}}
\toprule
\multicolumn{6}{c}{\textbf{SNGAN}}                                                                                                                                                                               \\ \midrule
                  & \multicolumn{1}{l}{$LR_D$} & \multicolumn{1}{l}{$LR_G$} & \begin{tabular}[c]{@{}c@{}}$Step$ \\ $Channels$\end{tabular} & \multicolumn{1}{l}{$\beta_1$} & \multicolumn{1}{l}{$\beta_2$} \\ \midrule
\textit{MNIST}    & $1e-4$                      & $2e-4$                      & $16$                                                       & $0.00$                                        & $0.999$                                       \\
\textit{CIFAR-10} & $1e-4$                      & $2e-4$                      & $64$                                                       & $0.00$                                        & $0.999$                                       \\
\textit{CELEB-A}  & $1e-4$                      & $2e-4$                      & $64$                                                       & $0.00$                                        & $0.999$                                       \\ \bottomrule
\end{tabular}
\caption{Hyperparameter values for SNGAN experiments.}
\label{hp_sngan}
\end{table}

\begin{table}[h]
\centering
\begin{tabular}{@{}lccccc@{}}
\toprule
\multicolumn{6}{c}{\textbf{WGAN}}                                                                                                                                                                               \\ \midrule
                  & \multicolumn{1}{l}{$LR_D$} & \multicolumn{1}{l}{$LR_G$} & \begin{tabular}[c]{@{}c@{}}$Step$ \\ $Channels$\end{tabular} & \multicolumn{1}{l}{$\beta_1$} & \multicolumn{1}{l}{$\beta_2$} \\ \midrule
\textit{MNIST}    & $4e-4$                      & $1e-4$                      & $16$                                                       & $0.50$                                        & $0.999$                                       \\
\textit{CIFAR-10} & $4e-4$                      & $1e-4$                      & $64$                                                       & $0.50$                                        & $0.999$                                       \\
\textit{CELEB-A}  & $4e-4$                      & $1e-4$                      & $64$                                                       & $0.50$                                        & $0.999$                                       \\ \bottomrule
\end{tabular}
\caption{Hyperparameter values for WGAN experiments.}
\label{hp_wgan}
\end{table}

\bibliographystyle{icml2021}
% \bibliography{main}